\documentclass{article} 

\usepackage{amsmath,amsfonts,bm}

\def\eqref#1{equation~\ref{#1}}

\def\1{\bm{1}}

\DeclareMathAlphabet{\mathsfit}{\encodingdefault}{\sfdefault}{m}{sl}
\SetMathAlphabet{\mathsfit}{bold}{\encodingdefault}{\sfdefault}{bx}{n}

\newcommand{\KL}{D_{\mathrm{KL}}}

\DeclareMathOperator*{\argmin}{arg\,min}

\usepackage{url}
\usepackage{amsmath,amssymb}
\usepackage{algorithm}
\usepackage{algorithmicx}
\usepackage{algpseudocode}
\usepackage{booktabs}
\usepackage{stmaryrd}
\usepackage{graphicx}
\usepackage{lipsum} 
\usepackage{amsthm}
\usepackage{wrapfig}
\usepackage{microtype}
\usepackage{subcaption}
\usepackage{aligned-overset}
\usepackage{multirow}
\usepackage{enumitem}
\usepackage[dvipsnames]{xcolor}
\usepackage{xurl}
\usepackage[
  colorlinks=true,
  linkcolor=NavyBlue,
  citecolor=BlueViolet,
  urlcolor=violet,
  pdfpagemode=UseOutlines,
  bookmarks=true,
  bookmarksopen=true,
  bookmarksnumbered=true,
  bookmarksdepth=3,
]{hyperref}

\let\ICMLSavedAddContentsLine\addcontentsline

\usepackage[accepted]{icml2026}

\let\addcontentsline\ICMLSavedAddContentsLine

\hypersetup{
  colorlinks=true,
  linkcolor=NavyBlue,
  citecolor=BlueViolet,
  urlcolor=violet,
  pdfpagemode=UseOutlines,
}

\newcommand{\Input}{\State\textbf{Input:} }
\newcommand{\Output}{\State\textbf{Output:} }

\DeclareMathOperator{\proj}{proj}

\newtheorem{definition}{Definition}[section]
\newtheorem{theorem}{Theorem}[section]
\newtheorem{lemma}{Lemma}[section]
\newtheorem{proposition}{Proposition}[section]

\newtheorem{conjecture}[theorem]{Conjecture}

\icmltitlerunning{Multi-marginal temporal Schrödinger Bridge Matching from unpaired data}

\begin{document}

\twocolumn[
  \icmltitle{Multi-marginal temporal Schrödinger Bridge Matching from unpaired data}



  \icmlsetsymbol{equal}{*}

  \begin{icmlauthorlist}
    \icmlauthor{Thomas Gravier}{equal,ibens,ps}
    \icmlauthor{Thomas Boyer}{equal,ibens}
    \icmlauthor{Auguste Genovesio}{ibens}
  \end{icmlauthorlist}

  \icmlaffiliation{ibens}{IBENS, ENS Ulm, PSL, Paris, France}
  \icmlaffiliation{ps}{ENS Paris-Saclay, Paris, France}

  \icmlcorrespondingauthor{Auguste Genovesio}{auguste.genovesio@bio.ens.psl.eu}


  \vskip 0.2in
]

\printAffiliationsAndNotice{\icmlEqualContribution}

\begin{abstract}
Many natural dynamic processes -such as in vivo cellular differentiation or disease progression- can only be observed through the lens of static sample snapshots. While challenging, reconstructing their temporal evolution to decipher underlying dynamic properties is of major interest to scientific research. Existing approaches enable data transport along a temporal axis but are poorly scalable in high dimension and require restrictive assumptions to be met. To address these issues, we propose \textit{\textbf{Multi-marginal temporal Schrödinger Bridge Matching}} (\textbf{MMtSBM}) \textit{from unpaired data}, extending the theoretical guarantees and empirical efficiency of Diffusion Schrödinger Bridge Matching \citep{shi2023diffusionschrodingerbridgematching} by deriving the Iterative Markovian Fitting algorithm to multiple marginals in a novel factorized fashion. Experiments show that MMtSBM retains theoretical properties on toy examples, achieves state-of-the-art performance on real-world datasets such as transcriptomic trajectory inference in 100 dimensions, and for the first time recovers couplings and dynamics in very high dimensional image settings, effectively generating temporally coherent videos from purely unpaired data. Our work establishes multi-marginal Schrödinger Bridges as a practical and theoretically principled approach for recovering hidden dynamics from static data.\\[1mm]
{\small code: \href{https://github.com/tgravier/MMDSBM-pytorch}{github.com/tgravier/MMDSBM-pytorch}\\website: \href{https://mmtsbm.notion.site}{mmtsbm.notion.site}}
\end{abstract}

\section{Introduction}
The observation of many natural processes yields partial information, resulting in limited time resolution and unpaired snapshots of data. Common examples of this are single-cell sequencing and in vivo biological imaging, where existing methods are destructive and thus cannot link two observations coming from the same cell at different timestamps. The ability to recover the true underlying dynamic from time-unpaired data samples is a key motivation for developing improved methods of trajectory inference.

The modeling of this problem is inherently probabilistic, given both the variability occurring in complex natural processes and the uncertainty of the observation. We thus ask the question: \textit{"What is the most probable evolution of an existing data point, given uncoupled samples of the same process acquired across different times?"}.\\
This point of view has notably been developed in the Schrödinger Bridge (SB) theory \citep{schrodinger1931}. The SB is the unique stochastic process whose marginals at start and end times match given probability distributions while minimizing the Kullback–Leibler (KL) divergence w.r.t. a given reference process. The SB also happens to solve a regularized Optimal Transport (OT) problem \citep{leonard2014survey}. Some recent works such as \citet{mmsb, baradat2020minimizingrelativeentropypath, lavenant} have explored the theoretical setting of multiple marginals. Recent major advances in statistical learning of SBs have allowed using this framework between complex empirical distributions \citep{debortoli2021diffusionschrodingerbridgeapplications, wang2021deepgenerativelearningschrodinger}, achieved important improvements in their efficiency \citep{shi2023diffusionschrodingerbridgematching, debortoli2024schrodingerbridgeflowunpaired}, extended it to the multi-marginal setting and explored various additional constraints such as smooth trajectories \citep{chen2023deepmomentummultimarginalschrodinger, hong2025trajectoryinferencesmoothschrodinger}, and spline-valued trajectories \citep{ theodoropoulos2025momentummultimarginalschrodingerbridge}. A few methods have been proposed to solve the SB problem in an applied machine learning setting. \citet{debortoli2021diffusionschrodingerbridgeapplications} use iterative proportional fitting (IPF) \citep{kullback}, the general continuous analogue of the renowned Sinkhorn algorithm \citep{cuturi2013sinkhorn}. Subsequent works have explored alternative training schemes based on likelihood bounds \citep{chen2023likelihoodtrainingschrodingerbridge} or on the dual algorithm of IPF: Iterative Markovian Fitting (IMF) \citep{shi2023diffusionschrodingerbridgematching}.\\
A closely related line of work is flow matching \citep{lipman2023flowmatchinggenerativemodeling, liu2022flowstraightfastlearning, albergo2023buildingnormalizingflowsstochastic}. These methods have explored OT variants since their inception and have been extended to the multi-marginal setting as well as connected to the Schrödinger Bridge theory \citep{tong2024improvinggeneralizingflowbasedgenerative, tong2024simulationfreeschrodingerbridgesscore, kapuśniak2024metricflowmatchingsmooth}.\\
Concurrent to our work is \citet{park2025multimarginalschrodingerbridgematching}; we note that they do not scale to video experiments. 

Existing multi-marginal methods do not scale to very high dimensions such as image space. Furthermore we believe that existing multi-marginal approaches either make use of modeling assumptions that strongly restrict the class of problems they can solve, such as using spline-valued trajectories, or lack a fully theoretically sound approach.

\paragraph{Contributions} We make the following contributions:
\begin{enumerate}[leftmargin=5mm]
    \item We define the multi-marginal temporal Schrödinger Bridge problem and demonstrate its fundamental properties (existence and uniqueness of the solution).
    \item We introduce a novel factorized extension of the IMF algorithm presented in \citet{shi2023diffusionschrodingerbridgematching} to multiple iterative marginals in a way that is efficient –because parallelized along times, and principled –because mathematically sound and with a concrete algorithm closely following theory.
    \item We produce a convergence analysis of the algorithm under asymptotic hypotheses.
    \item We demonstrate the soundness of the method on low-to-medium-dimensional examples, and achieve state-of-the-art results against comparable methods on $2$ widely reported single-cell transcriptomic benchmarks \citep{Moon2019, open-problems-multimodal}.
    \item  We scale up to 7 iterative marginals in a very high-dimensional \textit{image} setting, presenting for the first time a coherent video generation algorithm from purely time-unpaired data samples. 
\end{enumerate}

\paragraph{Notations} We adopt the notations from \citet{shi2023diffusionschrodingerbridgematching}. We denote by $\mathcal{P}(C)$ the space of \emph{path measures}, with $\mathcal{P}(C)=\mathcal{P}(C([0,T],\mathbb{R}^d))$, where $C([0,T],\mathbb{R}^d)$ is the space of continuous functions from $[0,T]$ to $\mathbb{R}^d$. The subset of \emph{Markov path measures} associated with the diffusion $dX_t=v_t(X_t)dt+\sigma_t dB_t$, with $\sigma,v$ locally Lipschitz, is denoted $\mathcal{M}$. We denote $(B_t)_{t\ge0}$ the $d$-dimensional Brownian motion. For a process $\mathbb{Q}$, the \emph{reciprocal class} of $\mathbb{Q}$ is $\mathcal{R}(\mathbb{Q})$. For $\mathbb{P}\in\mathcal{P}(C)$, we denote by $\mathbb{P}_t$ its marginal at time $t$, by $\mathbb{P}_{s,t}$ the joint law at times $s,t$, and by $\mathbb{P}_{s|t}$ the conditional law at $s$ given $t$. We write $\mathbb{P}_{|t_i,t_j}\in\mathcal{P}(C)$ for the path distribution on $(t_i,t_j)$ given the endpoints $t_i$ and $t_j$; e.g., $\mathbb{Q}_{|t_i,t_j}$ is a scaled Brownian bridge. Unless otherwise specified, $\nabla$ refers to gradients w.r.t. $x_t$ at time $t$. For a joint law $\Pi_{0,T}$ on $\mathbb{R}^d\times\mathbb{R}^d$, the \emph{mixture of bridges measure} is $\Pi=\Pi_{0,T}\mathbb{P}_{|0,T}\in\mathcal{P}(C)$ with $\Pi(\cdot)=\int_{\mathbb{R}^d\times\mathbb{R}^d}\mathbb{P}_{|0,T}(\cdot|x_0,x_T)d\Pi_{0,T}(x_0,x_T)$. The entropy of a process w.r.t. the Brownian motion is denoted $\mathcal{H}$. Finally, for $\pi_0,\pi_T\in\mathcal{P}(X)$, the Kullback--Leibler divergence is $\mathrm{KL}(\pi_0\|\pi_T)=\int_X \log\!\left(\tfrac{d\pi_0}{d\pi_T}(x)\right)\, d\pi_0(x)$.

\section{Background}

\subsection{The Schrödinger Bridge problem}
The \emph{Schrödinger Bridge problem}~\citep{schrodinger1931} seeks the most 
likely stochastic evolution between marginals $\mu_0,\mu_T$ under a reference law 
$\mathbb{Q}$. It admits both a \emph{dynamic} formulation:
\begin{equation}
    \label{eqn:dynSB}
    \mathbb{P}^\star = 
    \argmin_{\mathbb{P}\in\mathcal{P}(C)} 
    \mathrm{KL}(\mathbb{P}\,\|\,\mathbb{Q})
    \;\; \text{s.t. } \mathbb{P}_0=\mu_0, \;\mathbb{P}_T=\mu_T,
\end{equation}
and a \emph{static} formulation on couplings $\Pi\in\mathcal{P}(\mathbb{R}^d\times\mathbb{R}^d)$:
\begin{equation}
    \begin{split}
        \Pi^\star = \argmin_{\Pi} \mathrm{KL}(\Pi\,\|\,\mathbb{Q}_{0,T})\\
        \text{s.t. } \Pi_0=\mu_0, \;\Pi_T=\mu_T.
    \end{split}
    \label{eqn:statSB}
\end{equation}

\paragraph{Note: Connection to Quadratic OT.}  
If $\mathbb{Q}$ is Brownian motion, \eqref{eqn:statSB} is precisely 
entropy-regularized quadratic OT with cost $c(x_0,x_T)=\tfrac{1}{2}\|x_0-x_T\|^2$ 
and regularization $\varepsilon=\sigma^2$. In the limit $\varepsilon\to 0$, this 
recovers classical OT, which motivates our interpolation framework.

\subsection{Iterative Markovian Fitting (IMF)}
The SB solution is the unique path measure that is both 
\emph{Markovian} and belongs to the \emph{reciprocal class} of $\mathbb{Q}$ 
while matching marginals~\citep{leonard2014survey}.  
This motivates the \emph{Iterative Markovian Fitting} (IMF) algorithm 
\citep{shi2023diffusionschrodingerbridgematching, Peluchetti2023DiffusionBM}, 
which alternates between reciprocal and Markov projections:
\begin{equation}
    \mathbb{P}^{2n+1}=\proj_{\mathcal{M}}(\mathbb{P}^{2n}), \;
    \mathbb{P}^{2n+2}=\proj_{\mathcal{R}(\mathbb{Q})}(\mathbb{P}^{2n+1})
\end{equation}
These projections admit KL variational characterizations 
(\ref{app:background}) and the iterations converge to $\mathbb{P}^\star$.  

In practice, IMF is implemented by learning the drift of the Markovian 
projection via a bridge-matching loss (see \ref{app:background}). 
Compared to Iterative Proportional Fitting (IPF), IMF preserves both marginals 
simultaneously and is more efficient (details in \ref{app:background}).

\citet{shi2023diffusionschrodingerbridgematching} succeeds in scaling to image space but is limited to two marginals only. To the best of our knowledge, no theoretical study of their IMF algorithm in non-phase-lifted space has been made in the multi-marginal case, and the scaling capabilities of the algorithm in this setting remain unknown.

\section{Multi-marginal temporal Schrödinger Bridge Matching}
We start by exposing known theoretical results about multi-marginal Schrödinger Bridges (\ref{sec:mmtsbm-problem}).\footnote{We originally believed Proposition~\ref{prop:form-of-the-solution} to be novel, however it appears to be already established in \citet{baradat2020minimizingrelativeentropypath}.} Then, to the best of our knowledge, from \ref{sec:IMF-for-mmtsbm} onward we present novel results. All proofs can be found in the Appendix \ref{sec:proofs}.

\subsection{Multi-marginal temporal Schrödinger Bridge Problem}\label{sec:mmtsbm-problem}
In the present work, we considered the time-ordered Multi-marginal Schrödinger Bridge, where the marginals are associated with an underlying temporal axis. 
In this setting, the goal is not simply to fit an arbitrary number of marginals, but to recover the law of a stochastic process that evolves consistently over time.

Let $0 = t_0 < t_1 < \cdots < t_K = T$ be a fixed time grid, and let $\mu_0, ..., \mu_k, ..., \mu_T \in \mathcal{P}(\mathbb{R}^d)$ denote prescribed marginals at times $(t_k)_{k=0,\dots,K}$, assuming $\mu_{t_k} \ll \mathbb{Q}_{t_k}$ for all $k$. Given a reference process $\mathbb{Q}$ on $C([0,T],\mathbb{R}^d)$, 
the multi-marginal Schrödinger Bridge problem (MMSB) is defined as
\begin{equation}
\begin{gathered}
    \mathbb{P}^\star \;=\; \operatornamewithlimits{argmin}_{\mathbb{P}\in\mathcal{P}(C)} \,\mathrm{KL}\!\left(\mathbb{P}\,\|\,\mathbb{Q}\right)\\
    \text{subject to} \quad X_{t_k}\sim \mu_k, \;\; k=0,\dots,K
\end{gathered}
\label{eqn:MMSB}
\end{equation}

\paragraph{Note: Connection to multi-marginal Optimal Transport}
If $\mathbb{Q}$ is associated with a Brownian motion, the induced reference coupling $\mathbb{Q}_{t_0,\dots,t_K}$ is characterized by independent Gaussian increments $X_{t_{i+1}} - X_{t_i} \sim \mathcal{N}(0, \sigma^2 (t_{i+1}-t_i))$. By evaluating the KL term, \ref{eqn:MMSB} can be rewritten as:
\begin{equation*}
\begin{gathered}
    \Pi^{\star} = \operatornamewithlimits{argmin}_{\substack{\Pi \in \mathcal{P}\left((\mathbb{R}^d)^{K+1}\right) \\ \Pi_i = \mu_{t_i}}} \Bigg\{ \mathbb{E}_{X \sim \Pi} \Bigg[ \sum_{i=0}^{K-1} c_i(X) \Bigg] - 2\sigma^2 T \mathcal{H}(\Pi)\Bigg\} \\
    \text{where}\quad c_i(X) = \frac{1}{t_{i+1}-t_i}\,\|X_{t_{i+1}} - X_{t_i}\|^2
\end{gathered} 
\end{equation*}
This is precisely an entropy-regularized multi-marginal OT problem with a time-structured quadratic cost
$c(x_0,\dots,x_K) = \sum_{i=0}^{K-1} \frac{1}{t_{i+1}-t_i}\|x_{i+1} - x_i\|^2$ and entropy-regularisation parameter $\varepsilon = 2\sigma^2$.

This formulation is particularly interesting when no better prior is available, and because of the clear interpretation it allows: when using a Brownian motion as prior, we are approaching quadratic OT. Note however that we do not rely on this assumption at all for theoretical results.

\paragraph{Classical properties of the Multi-marginal temporal Schrödinger Bridge}
We first demonstrate a set of classical properties that characterize MMSB (\ref{eqn:MMSB}) and guide the construction of our method. 

\begin{definition}[Static formulation]\label{def:SMSB}
Let $\mathbb{Q}_{t_0,\dots,t_K}$ be the joint law of $\mathbb{Q}$ at 
$0=t_0<\cdots<t_K=T$. 
The static problem is
\[
\pi^\star = \arg\min_{\pi \in \Pi(\pi_{t_0},\dots,\pi_{t_K})} 
KL(\pi \,\|\, \mathbb{Q}_{t_0,\dots,t_K}),
\]
where $\Pi(\pi_{t_0},\dots,\pi_{t_K})$ denotes couplings on 
$(\mathbb R^d)^{K+1}$ with marginals $\pi_{t_i}$.
\end{definition}

The MMSB is therefore a projection of the reference law onto the set of 
couplings with prescribed marginals. The following results ensure that 
this problem is well posed and that the solution has a convenient structure.

\begin{proposition}[Existence and uniqueness]\label{prop:existence_uniqueness}
The MMSB admits a unique solution $P^\star$.
\end{proposition}

This guarantees that the iterative algorithms we design later target a 
well-defined object. Moreover, the solution can be described equivalently 
in both static and dynamic terms. 

\begin{proposition}[Dynamic--static equivalence]\label{prop:dyn_static_equiv}
The dynamic solution $P^\star$ is determined by the static one $\pi^\star$:
\[
\pi^\star = P^\star_{t_0,\dots,t_K}, 
\quad 
P^\star = \pi^\star \otimes \mathbb{Q}(\cdot \mid X_{t_0},\dots,X_{t_K}).
\]
\end{proposition}

This equivalence highlights that solving the static problem is enough to 
recover the full path measure. In addition, the structure of $\mathbb{Q}$ plays a 
key role in the nature of the solution.

\begin{proposition}[Markovianity]\label{prop:markovianity}
If $\mathbb{Q}$ is Markov, then the MMSB solution $P^\star$ is Markov.
\end{proposition}

These properties ensure that we can restrict our search to Markovian (and 
therefore reciprocal \ref{prop:markov_implies_reciprocal}) measures, which will be central to the projection 
algorithms introduced later. Finally, the explicit form of the solution 
further clarifies its structure.

\begin{proposition}[Form of the solution]\label{prop:form-of-the-solution}
Under mild assumptions:
\begin{gather*}
    P^\star = \pi^\star \otimes \mathbb{Q}(\cdot \mid X_{t_0},\dots,X_{t_K}),\\
    \frac{d\pi^\star}{d\mathbb{Q}_{t_0,\dots,t_K}}(x_0,\dots,x_K) = \prod_{i=0}^K f_i(x_i).
\end{gather*}
\label{prop:form_dynamic}
\end{proposition}
where the $f_i$'s are functions of the Lagrange multipliers for the marginal constraints (see \ref{app:proof_form_dynamic}).\\
This factorized form motivates the use of alternating projections and parametric families of potentials in the iterative algorithm that we develop in the next section.

\subsection{Iterative Markovian Fitting for Multi-marginal temporal Schrödinger Bridge}\label{sec:IMF-for-mmtsbm}
\subsubsection{Multi-marginal Markov and Reciprocal projections}
To construct an algorithm for MMSB, we first extend the notions of reciprocal and Markovian projections to the multi-marginal setting. The idea is to approximate the global bridge by a sequence of independent sub-bridges, and to alternate between reciprocal and Markovian structures. However, a naive implementation of this idea leads to model forgetting, failing to converge entirely. A factorized approach is thus necessary to obtain convergence (see \ref{sec:crit_implem}).

\begin{definition}[Factorized reciprocal class and projection]
\label{def:factorized_reciprocal_class}
For each interval $[t_i,t_{i+1}]$ and endpoints $(x_i,x_{i+1})$, 
let $\mathbb{Q}^{x_i,x_{i+1}}_{[t_i,t_{i+1}]}$ denote the bridge of $\mathbb{Q}$ between 
$x_i$ and $x_{i+1}$. 
Given a coupling $\pi$ on $(\mathbb R^d)^{K+1}$, define
\[
P = \int \bigotimes_{i=0}^{K-1} \mathbb{Q}^{x_i,x_{i+1}}_{[t_i,t_{i+1}]} \, 
\pi(dx_0,\dots,dx_K).
\]
The \emph{factorized reciprocal class}, denoted $\mathcal R^{\otimes}(\mathbb{Q})$, 
is the set of all such measures $P$. 

Moreover, for any $P \in \mathcal P(C([0,T],\mathbb R^d))$, 
the \emph{reciprocal projection} onto $\mathcal R^\otimes(\mathbb{Q})$ is defined as
\[
\Pi^\star = \proj_{\mathcal R^\otimes(\mathbb{Q})}(P) 
= P_{t_0,\dots,t_K}\, \bigotimes_{i=0}^{K-1} 
\mathbb{Q}^{x_i,x_{i+1}}_{[t_i,t_{i+1}]},
\]
i.e. we keep the marginals $P_{t_0,\dots,t_K}$ at the grid points and fill 
the dynamics between them with independent bridges of $\mathbb{Q}$ conditioned on 
the endpoints $(x_i,x_{i+1})$.

Equivalently, $\Pi^\star$ admits the variational characterization
\[
\Pi^\star = \argmin_{\Pi \in \mathcal R^\otimes(\mathbb{Q})} KL(P \,\|\, \Pi).
\]
\end{definition}

\begin{proposition}[Local reciprocal structure of the factorized class]\label{prop:local_reciprocal}
Let $\mathbb{Q}$ be a reference Markov process and let 
$P \in \mathcal R^{\otimes}(\mathbb{Q})$ belong to the factorized reciprocal class. 
Then, for each subinterval $[t_{i-1},t_i]$, 
the restriction of $P$ to $C([t_{i-1},t_i],\mathbb R^d)$ is in the 
reciprocal class of $\mathbb{Q}$ over $[t_{i-1},t_i]$. 
In particular, conditionally on the endpoints 
$(X_{t_{i-1}},X_{t_i})$, the law of $P$ coincides with the bridge of $\mathbb{Q}$ 
between $t_{i-1}$ and $t_i$.
\end{proposition}

This class provides a tractable approximation: each sub-interval is 
filled with the bridge of $\mathbb{Q}$, while the global coupling ensures 
consistency across marginals.
Hence, factorized bridges inherit local reciprocity, which justifies 
their use as a relaxation of the true reciprocal class. 

This projection enforces the prescribed marginals while completing the dynamics with local bridges. In contrast, the Markovian projection seeks a single Markov diffusion with consistent marginals.

\begin{definition}[Markovian projection in the factorized setting]\label{def:markov_projection_factorized}
Let $\Pi$ be the factorized mixture of independent Brownian bridges. 
For any $t \in [0,T]$, let $i(t)$ be the unique index such that $t \in [t_{i(t)},t_{i(t)+1}]$. We employ a slight abuse of notation and subsequently write $i$ instead of $i(t)$.

The \emph{Markovian projection} of $\Pi$, denoted $M^\star=\proj_{\mathcal M}(\Pi)$, 
is the unique diffusion process
\[
dX_t^\star 
= \big\{ f_t(X_t^\star) + v_t^\star(X_t^\star)\big\}\,dt 
+ \sigma_t \, dB_t,
\]
with effective drift
\begin{align*}
    v_t^\star(x) &= \sigma_t^2 \,\mathbb E_{\Pi_{t_{i+1}\mid t}}
    \!\left[
      \nabla \log \mathbb{Q}^{\,|t_{i},t_{i+1}}_{t}(X_{t_{i+1}} \mid X_t) 
      \;\middle|\; X_t = x
    \right] \\
    \overset{\text{Brownian}}&{=}
    \frac{\mathbb{E}_{\Pi_{t_{i+1}\mid t}}[X_{t_{i+1}} \mid X_t = x] - x}{t_{i+1}-t}
\end{align*}

By the Markovian projection theorem of \citet{gyongy1986}, and as further developed in \citet{Peluchetti2023DiffusionBM,debortoli2021diffusionschrodingerbridgeapplications}, 
the process $M^\star$ is Markov and matches the one-dimensional marginals of the original factorized law $\Pi$.
\end{definition}

\begin{proposition}[Variational characterization of the factorized Markovian projection]\label{prop:var_proj_factorized}
Assume that $\sigma_t > 0$. 
Let $M^\star = \proj_{\mathcal M}(\Pi)$ be the Markovian projection of $\Pi$ as in Definition~\ref{def:markov_projection_factorized}. 
Then:
\[
M^\star = \arg\min_{M \in \mathcal M} \big\{ KL(\Pi \,\|\, M) \big\},
\]
and
\begin{equation*}
    KL(\Pi \,\|\, M^\star) =
    \frac{1}{2} \int_0^T 
    \mathbb E_{\Pi_{t_{i} , t}}\!\left[
      \frac{1}{\sigma_t^2} 
      \Big\| V - v_t^\star(X_t) \Big\|^2 
    \right] dt
\end{equation*}
with:
\begin{equation*}
  V = \sigma_t^2\,\mathbb E_{\Pi_{t_{i+1}\mid t}}
  \!\big[ \nabla \log \mathbb{Q}^{\,|t_{i},t_{i+1}}_{t}(X_{t_{i+1}} \mid X_t) 
  \,\big|\, X_t, X_{t_i} \big]
\end{equation*}

In addition, for any $t \in [0,T]$, the time marginal of $M^\star$ coincides with that of $\Pi$: $M^\star_t = \Pi_t$. In particular, $M^\star_{t_i} = \Pi_{t_i}$ for all grid points $t_i$.
\end{proposition}

Together, these results allow us to alternate between reciprocal and 
Markovian structures in the multi-marginal setting. Importantly, the 
Markovian projection admits explicit forward and backward formulations.

\begin{proposition}\label{prop:reverse_sde}
Let $\Pi \in \mathcal{R}^\otimes(\mathbb{Q})$. Under mild regularity conditions, the Markovian projection 
$M^\star = \mathrm{proj}_{\mathcal{M}}(\Pi)$ is associated with the forward SDE
\begin{equation}
    dX_t = \Big\{ f_t(X_t) 
    + D_f \Big\}\, dt 
    + \sigma_t dB_t,\;
    X_{t_{i}} \sim \mu_{t_{i}}
\label{eq:multi_forward}
\end{equation}
and with the backward SDE
\begin{equation}
    dY_t = \Big\{ - f_{t_{i+1}-t}(Y_t) 
    + D_b \Big\}\, dt
    + \sigma_{t_{i+1}-t} dB_t, \;
    Y_{t_{i+1}} \sim \mu_{t_{i+1}}
\label{eq:multi_backward}
\end{equation}
where:
\begin{align*}
    D_f &= \sigma_t^2 \,\mathbb{E}_{\Pi_{t_{i+1}}|t}\big[ \nabla \log \mathbb{Q}^{[t_{i},t_{i+1}]}_{t}(X_{t_{i+1}} \mid X_t) \,\big|\, X_t \big] \\
    D_b &= \sigma_{t_{i+1}-t}^2 \,\mathbb{E}_{\Pi_{t_{i}}|t}\big[ \nabla \log \mathbb{Q}^{[t_{i},t_{i+1}]}_{t}(Y_{t_{i}} \mid Y_t) \,\big|\, Y_t \big]
\end{align*}
\end{proposition}
This key result highlights that the Markovian projection can be expressed both in the forward and in the backward direction, allowing us to design an algorithm that jointly leverages both dynamics.

\begin{conjecture}[Analogue of \citet{leonard2014survey} Theorem 2.12]\label{conj:markov_uniqueness_mmsb}
Let $\mathbb{Q}$ be a Markov reference process. 
Suppose that $P$ is a Markov path measure such that 
\[
P \in \mathcal R(\mathbb{Q}),
\qquad 
P_{t_i} = \mu_{t_i}, \quad i=0,\dots,K.
\]
Then $P$ coincides with the unique solution $P^\star$ of the 
multi-marginal Schrödinger Bridge problem (MMSB) with reference $\mathbb{Q}$.
\end{conjecture}

\begin{algorithm*}[t!]
\caption{Iterative Markovian Factorized Fitting (IMFF)}
\begin{algorithmic}[1]
  \State \textbf{Input:} time grid $0=t_0<\dots<t_K=T$, marginals $(\mu_{t_i})_{i=0}^K$, reference process $\mathbb{Q}$, number of iterations $N$
  \State \textbf{Init:} choose $\mathbb{P}^0 \in \mathcal{R}^\otimes(Q)$ with $\mathbb{P}^0_{t_i}=\mu_{t_i}$ for all $i$
  \For{$n = 0,\ldots,N-1$}
    \State \textbf{Backward Markovian step:} learn drift $v_\phi$ via SDE~\eqref{eq:multi_backward}, 
           yielding $\mathbb{P}^{2n+1}$ with $t_i$ updated and $t_{i+1}$ from $\mu_{t_{i+1}}$.
    \State \textbf{Forward reciprocal projection:}
           $\mathbb{P}^{2n+1} \gets \proj_{\mathcal{R}^\otimes(Q)}(\mathbb{P}^{2n+1})$ 
           (cf. Def.~\ref{def:factorized_reciprocal_class}), 
           filling bridges with $\mathbb{Q}$ using $t_i$ from $\mathbb{P}^{2n+1}$ and $t_{i+1}$ from the dataset.
    \State \textbf{Forward Markovian step:} learn drift $v_\theta$ via SDE~\eqref{eq:multi_forward}, 
           yielding $\mathbb{P}^{2n+2}$ with $t_{i+1}$ updated and $t_i$ from $\mu_{t_i}$.
    \State \textbf{Backward reciprocal projection:}
           $\mathbb{P}^{2n+2} \gets \proj_{\mathcal{R}^\otimes(Q)}(\mathbb{P}^{2n+2})$ 
           (cf. Def.~\ref{def:factorized_reciprocal_class}), 
           filling bridges with $\mathbb{Q}$ using $t_{i+1}$ from $\mathbb{P}^{2n+2}$ and $t_i$ from the dataset.
  \EndFor
  \State \textbf{Output:} learned drifts $(v_\phi, v_\theta)$
\end{algorithmic}
\label{alg:IMFF}
\end{algorithm*}

\subsubsection{Iterative Markovian Factorized Fitting}

Based on Conjecture~\ref{conj:markov_uniqueness_mmsb}, we propose a novel algorithm called \emph{Iterative Markovian Factorized Fitting} (IMFF) to solve multi-marginal Schrödinger Bridges. We consider a sequence 
$(\mathbb{P}^n)_{n\in\mathbb{N}}$ such that
\begin{equation}
    \mathbb{P}^{2n+1} = \proj_{\mathcal{M}}(\mathbb{P}^{2n}),
    \; 
    \mathbb{P}^{2n+2} = \proj_{\mathcal{R}^{\otimes}(\mathbb{Q})}(\mathbb{P}^{2n+1})
    \label{eq:IMFF}
\end{equation}
with $\mathbb{P}^0$ such that 
$\mathbb{P}^0_{t_i} = \mu_{t_i}$ for all $i=0,\dots,K$, 
and $\mathbb{P}^0 \in \mathcal{R}^{\otimes}(\mathbb{Q})$. 
These updates correspond to alternatively performing Markovian projections and factorized reciprocal projections in order to enforce all prescribed marginals.

\begin{lemma}[Pythagorean identities in the factorized setting]\label{lem:pythagoras_factorized}
Under mild assumptions, if $M \in \mathcal M$, 
$\Pi \in \mathcal R^\otimes(\mathbb{Q})$ and $KL(\Pi \| M)<+\infty$, we have
\[
KL(\Pi \| M) 
= KL(\Pi \| \proj_{\mathcal M}(\Pi)) + KL(\proj_{\mathcal M}(\Pi) \| M)
\]

Similarly, if $KL(M \| \Pi)<+\infty$, we have
\begin{multline*}
    KL(M \| \Pi) = KL(M \| \proj_{\mathcal R^\otimes(\mathbb{Q})}(M)) \\
    + KL(\proj_{\mathcal R^\otimes(\mathbb{Q})}(M) \| \Pi)
\end{multline*}
\end{lemma}

\begin{proposition}\label{prop:convergence_projection}
Under mild assumptions, we have 
\begin{gather*}
    KL(P^{n+1}\,\|\,P^\star) \;\leq\; KL(P^n\,\|\,P^\star) < \infty \\[1mm]
    \lim_{n\to\infty} KL(P^n\,\|\,P^\star) = 0
\end{gather*}
\end{proposition}

Hence, for the IMFF sequence $(\mathbb{P}^n)_{n \in \mathbb{N}}$, 
the Markov path measures $(\mathbb{P}^{2n+1})_{n \in \mathbb{N}}$ are getting closer to 
the factorized reciprocal class $\mathcal{R}^{\otimes}(\mathbb{Q})$, while the reciprocal path measures 
$(\mathbb{P}^{2n+2})_{n \in \mathbb{N}}$ are getting closer to the set of Markov measures. 
This mirrors the situation in the classical IMF setting, but now in the multi-marginal framework.

\begin{theorem}
\label{thm:imff_convergence}
Under mild assumptions, the IMFF sequence $(\mathbb{P}^n)_{n \in \mathbb{N}}$ 
admits at least one fixed point $\mathbb{P}^\star$, and we have:
\[
\lim_{n \to +\infty} KL(\mathbb{P}^n \,\|\, \mathbb{P}^\star) = 0
\]
Moreover, denoting by $\mathbb{P}^{\mathrm{MMSB}}$ the solution of (\hyperref[eqn:MMSB]{MMSB}) and by $\mathbb{P}^{\mathrm{pair}}$ the gluing of pairwise Schrödinger Bridges, the limit of the IMFF sequence satisfies the inequality:
\[
KL(\mathbb{P}^{\mathrm{MMSB}} \,\|\, \mathbb{Q})
\; = \; KL(\mathbb{P}^\star \,\|\, \mathbb{Q})
\;\leq\; KL(\mathbb{P}^{\mathrm{pair}} \,\|\, \mathbb{Q})
\]
where $\mathbb{Q}$ is the chosen reference process. 
Thus, $\mathbb{P}^\star$ is the multi-marginal Schrödinger Bridge.
\end{theorem}

\subsubsection{Theoretical algorithm}
The Markovian projection necessitates learning one neural drift per direction. Concretely, we solve
\begin{equation}
    \operatornamewithlimits{argmin}_{\theta}
    \mathbb{E}_b\!\left[
    \big\| v_\theta(X_t,t) - 
    \sigma_t^2 \, \mathbb{E}\big[G_t(X_{t_{i+1}}) \,\big|\,X_t\big] 
    \big\|^2 \right]
    \label{eq:forward}
\end{equation}
for the \emph{forward} drift $v_\theta$, and
\begin{equation}
    \operatornamewithlimits{argmin}_{\phi} 
    \mathbb{E}_b\!\left[
    \big\| v_\phi(Y_t,t) - 
    \sigma_{t_{i+1}-t}^2 \, \mathbb{E}\big[G_t(Y_{t_i})\,\big|\,Y_t\big] 
    \big\|^2 \right]
    \label{eq:backward}
\end{equation}
for the \emph{backward} drift $v_\phi$, where $\mathbb{E}_b$ is over the batch and:
\begin{align*}
    G_t(Z) &= \nabla \log \mathbb{Q}^{[t_{i(t)},t_{i(t)+1}]}_t(Z \mid Z_t) \\
\end{align*}

We summarize in Algorithm \ref{alg:IMFF} our method and provide a practical implementation of IMFF in \ref{subsec:concrete_algo}.

\begin{proposition}
Suppose the families of functions $\{v_\theta : \theta \in \Theta\}$ and $\{v_\phi : \phi \in \Phi\}$ are rich enough to represent the optimal forward and backward drifts. Let $(P^n, M^n)_{n \in \mathbb{N}}$ be the sequence produced by Algorithm~1. Then, as $n \to \infty$, we have convergence towards an approximate multi-marginal Schrödinger Bridge. Moreover, the Markov law $M^n$ coincides in the limit with the intermediate approximate MMSB solution lying between the true multi-marginal Schrödinger Bridge and the pairwise construction.
\end{proposition}

\section{Experiments}
For all experiments, we employ Brownian motion $(\sigma_tB_t)_{0\le t \le T}$ for the reference measure $\mathbb{Q}$ and $T = N - 1$ where $N$ is the number of marginals. All trainings start after a warmup phase like in \citet{shi2023diffusionschrodingerbridgematching}, detailed in \ref{subsec:concrete_algo}. Videos for most experiments can be found at \href{https://mmtsbm.notion.site}{mmtsbm.notion.site}.

\subsection{MMtSBM recovers the exact OT between Gaussian mixtures}\label{sec:exact_ot}
In this 2D experiment akin to \citet{liu2022flowstraightfastlearning}, we used $N=3$ mixtures of two standard Gaussians as marginals. 
\begin{figure}[h]
    \centering
    \includegraphics[width=0.7\linewidth]{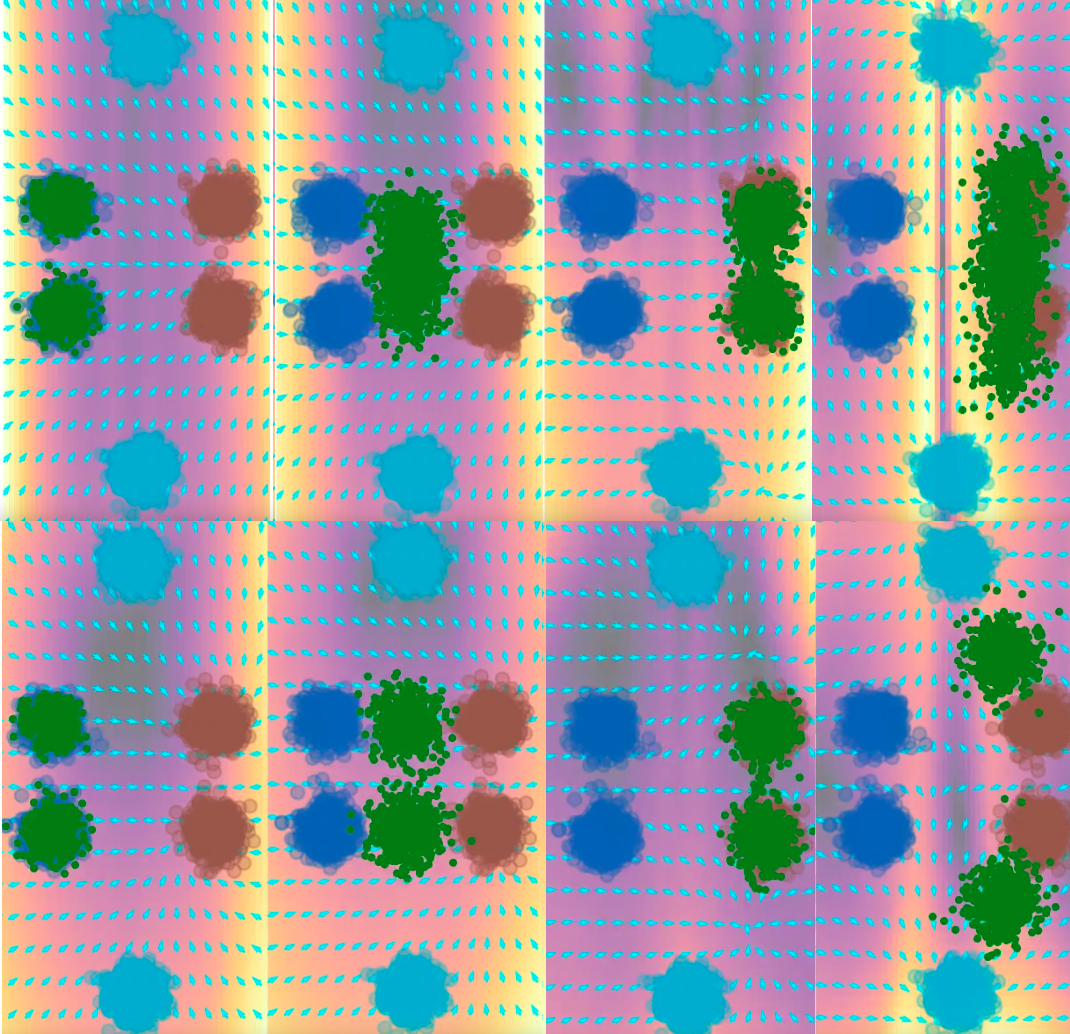}
    \caption{\small Top row: epoch $0$ (only noisy flow matching). Bottom row: epoch $5$ (after MMtSBM training). From left to right: snapshots at times $(0,0.5,1,1.3)$. True marginal times $(t_0, t_1, t_2)=(0, 1, 2)$. The order of the $3$ true marginals is: \textcolor[HTML]{0061B4}{$t_0=$ dark blue}; \textcolor[HTML]{995649}{$t_1=$ red}; \textcolor[HTML]{00ADCF}{$t_2=$ light blue}. \textcolor[HTML]{007C11}{Generated samples} are in \textcolor[HTML]{007C11}{green}. In the background is the quiver plot of the learned score network.}
    \label{fig:crossing_mixture}
\end{figure}

In this configuration the optimal transport \emph{between each pair of marginals} can be computed exactly: it is a pure translation of each Gaussian component inside the mixtures, as we verified with \href{https://pythonot.github.io/}{POT} (see Figure~\ref{fig:pot_exact_ot}). After only the warm-up phase (akin to flow matching \citep{lipman2023flowmatchinggenerativemodeling}), we can see that the learned transport maps \textit{mix} the Gaussian components of the mixtures, resulting in intersecting trajectories as can be seen in the top row of Figure~\ref{fig:crossing_mixture}. However, \textit{after} the SB learning phase of MMtSBM, we can see in the bottom row that the learned trajectories do \textit{not} intersect each other anymore and that MMtSBM yields the expected exact optimal transport map: pure translations between Gaussian components.\\
This observation is consistent with the theory: the warm-up phase preserves only the Markov property, while the final learned coupling additionally also preserves the reciprocal property, thus corresponding to the true SB. We empirically observe that the optimality emerges gradually along MMtSBM training epochs: trajectories get \textit{rectified} from epoch $1$, become optimal around epoch $5$, and consistently remain so after. We will now confirm these visual findings with quantitative metrics in \ref{sec:moons_8gauss}.

\subsection{MMtSBM achieves good usual SB metrics}\label{sec:moons_8gauss}
To quantitatively verify that MMtSBM recovers the correct multi-marginal SB in terms of both 1) static coupling and 2) energy minimization, we extended the now classical "Moons" and "8Gaussians" experiments found in \citet{tong2024improvinggeneralizingflowbasedgenerative} and \citet{shi2023diffusionschrodingerbridgematching} to our temporal multi-marginal setting in Table \ref{tab:moons_8gaussian} (see Figure \ref{fig:moons_8gaussian}). Choosing $N=4$, we considered ($\mathcal{N}$ $\to$ Moons $\to$ $\mathcal{N}$ $\to$ Moons), and ($\mathcal{N}$ $\to$ 8Gaussians $\to$ $\mathcal{N}$ $\to$ 8Gaussians). To assess 1) we report the $\mathcal{W}_2$ distance of generations vs test set data at target marginal time(s), averaging along the $N-1=3$ target times for MMtSBM and comparing this to the single bridge setting. To assess 2) we report the full path energy $\mathbb{E}\left[\int_{0}^{T} \| v(t, \mathbf{Z}_t) \|^2 \, dt \right]$ where $Z_t$ is the process simulated along the ODE drift \ref{eq:ODE_drift}.
\begin{table}[h]
    \centering
    \small
    \caption{\small Comparison in terms of static coupling ("$\mathcal{W}_2$") and energy minimization ("Path Energy"). The rows marked ``$\times 3$'' correspond to the hypothetical case where the energy of a single bridge is simply tripled, and are included as an ideal baseline for comparison with our actual multi-bridge setting. All metrics apart from ours are from \citet{shi2023diffusionschrodingerbridgematching}.}
    \begin{tabular}{clcc}
    \hline
     & Model & $\mathcal{W}_2$ & Path Energy \\
    \hline
    \multirow{3}{*}{\rotatebox{90}{Moons }} & Single bridge & $0.144 {\scriptstyle \pm 0.024}$ & $\textit{1.580} {\scriptstyle \pm \textit{0.036}}$ \\
    & Single bridge $\times 3$ & -- & $4.740$ \\
    & MMtSBM (ours) & $0.148 {\scriptstyle \pm 0.041}$ & $5.350 {\scriptstyle \pm 0.085}$ \\
    \hline\hline
    \multirow{3}{*}{\rotatebox{90}{8 $\mathcal{N}$}} & Single bridge & $0.338 {\scriptstyle \pm 0.091}$ & $\textit{14.810} {\scriptstyle \pm \textit{0.255}}$ \\
    & Single bridge $\times 3$ & -- & $44.430$ \\
    & MMtSBM (ours) & $0.352 {\scriptstyle \pm 0.084}$ & $46.920 {\scriptstyle \pm 0.285}$ \\
    \hline
    \end{tabular}
    \label{tab:moons_8gaussian}
\end{table}

We observe that despite a much more complex \emph{time-varying} true transport map to be learned, MMtSBM achieves almost as low $\mathcal{W}_2$ distances as the simple single-bridge setting (3\% to 4\%), and that our full path energy is within 13\% to 6\% of the ideal extrapolation of the single bridge result. This validates that MMtSBM manages to approach the true SB in practice.

\subsection{MMtSBM scales to $50d$ Gaussian transport}\label{sec:50d_gauss}
\begin{figure*}[t]
    \centering
        \includegraphics[trim=2mm 2mm 2mm 2mm, clip, width=0.25\linewidth]{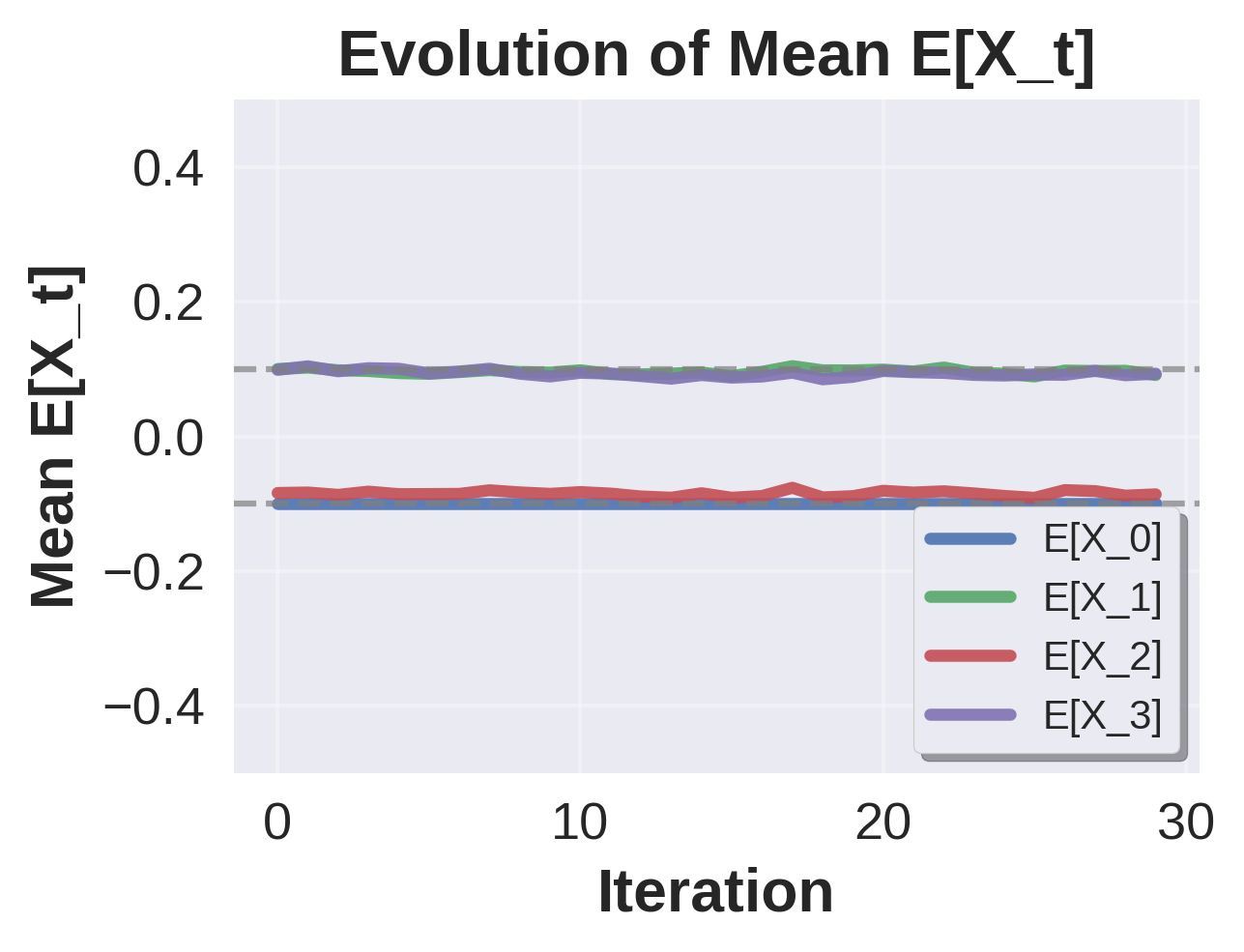}
        \includegraphics[trim=2mm 2mm 2mm 2mm, clip, width=0.25\linewidth]{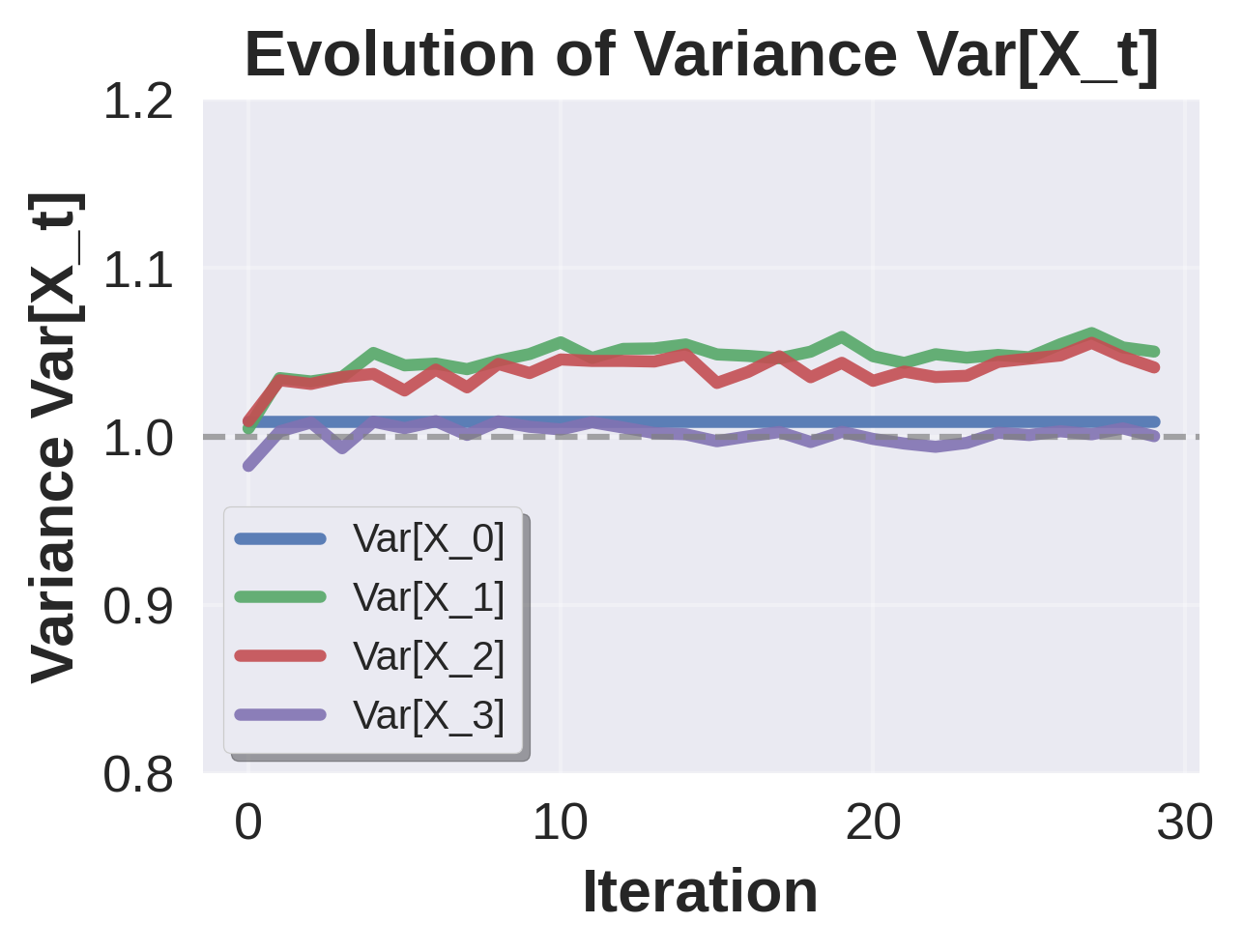}
        \includegraphics[trim=2mm 2mm 2mm 2mm, clip, width=0.25\linewidth]{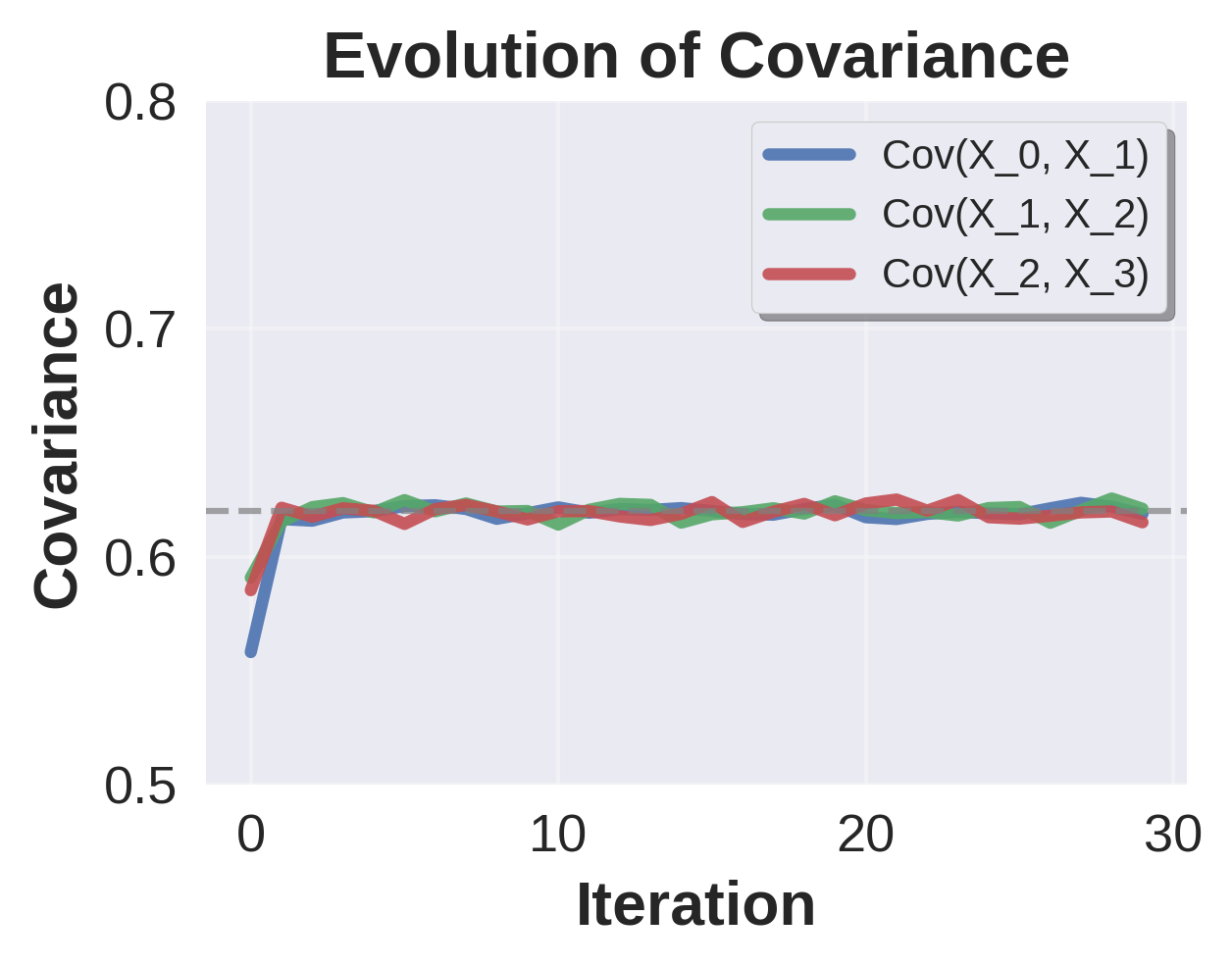}
        \caption{\small Evolution of mean, variance, and covariance in the multi-marginal $50d$ Gaussian case. Dashed lines are the theoretical values.}
        \label{fig:gaussian-evolution}
\end{figure*}
We next proceed to scaling our method to dimension $d=50$. We follow the setting of \citet{shi2023diffusionschrodingerbridgematching} and consider a Gaussian-to-Gaussian transport experiment, extended to our multi-marginal case. Specifically, we prescribe four Gaussian marginals at times $t=0,1,2,3$: $\mu_0 = \mbox{$\mathcal{N}(-0.1 \cdot \mathbf{1}_d, I_d)$}, \mu_1 = \mathcal{N}(0.1 \cdot \mathbf{1}_d, I_d), \mu_2 = \mathcal{N}(-0.1 \cdot \mathbf{1}_d, I_d), \mu_3 = \mathcal{N}(0.1 \cdot \mathbf{1}_d, I_d)$ where $\mathbf{1}_d \in \mathbb{R}^d$ denotes the vector of all ones, and $I_d$ is the $d \times d$ identity matrix.
Since no closed-form solution is available for the static multi-marginal SB, we compare our method to the sequence of theoretical results for each \emph{pairwise} SB \citep{bunne2022}. As shown in Figure~\ref{fig:gaussian-evolution}, the mean converges rapidly to the prescribed values ($0.1$ or $-0.1$) across all four marginals. The variance is slightly more difficult to match: for interior marginals the process tends to overestimate the standard deviation. In contrast, the covariance is consistently well reproduced by our method and remains stable across all three transitions. Interestingly, the covariance converges only after the warmup stage, confirming the added value of the subsequent OT phases. Overall, these results show that MMtSBM scales effectively to the multi-marginal Gaussian setting in $d=50$.

\subsection{MMtSBM achieves SOTA results on $100d$ transcriptomic benchmarks}\label{sec:100_transcriptomic}
\begin{table}[h]
    \small
    \setlength{\tabcolsep}{3pt}
    \centering
    \caption{\small MMD and SWD of generations vs test set for the $d=100$ EB benchmark. Our generations start from $\mu_{t=0}^\text{test}$. Top table: per-marginal metrics. Bottom table: average over all marginals. Others' results are from \citet{chen2023deepmomentummultimarginalschrodinger}. Our error margins are over 10 evaluations while DMSB's are over 3. Best value in \textbf{bold}.}
    \label{tab:embryoid-benchmark}
    \begin{tabular}{c|ll|ll}
        \toprule
             & \multicolumn{2}{c|}{DMSB \citep{chen2023deepmomentummultimarginalschrodinger}} & \multicolumn{2}{c}{\textbf{MMtSBM} (ours)} \\
        Time &               MMD $\downarrow$ & SWD $\downarrow$                              & MMD $\downarrow$ & SWD $\downarrow$ \\
        \midrule
        $t_1$ & 0.021 & 0.114 & \textbf{0.016} & \textbf{0.104} \\
        $t_2$ & 0.029 & 0.155 & \textbf{0.020} & \textbf{0.139} \\
        $t_3$ & 0.038 & 0.190 & \textbf{0.020} & \textbf{0.127} \\
        $t_4$ & 0.034 & 0.155 & \textbf{0.020} & \textbf{0.143} \\
        \midrule
        Average & 0.032 ${\scriptstyle \pm 3e-3}$ & 0.160 ${\scriptstyle \pm 2e-2}$ & \textbf{0.019} ${\scriptstyle \pm 4e-4}$ & \textbf{0.130} ${\scriptstyle \pm 2e-3}$ \\
        \bottomrule
    \end{tabular}
    
    \vspace{1mm}
    
    \begin{tabular}{l|ll}
        \toprule
        Algorithm & MMD $\downarrow$ & SWD $\downarrow$ \\
        \midrule
        NLSB \citep{koshizuka2023neurallagrangianschrodingerbridge}          & 0.66                                     & 0.54                                     \\
        MIOFlow \citep{huguet2022manifoldinterpolatingoptimaltransportflows} & 0.23                                     & 0.35                                     \\
        DMSB \citep{chen2023deepmomentummultimarginalschrodinger}            & 0.032 ${\scriptstyle \pm 3e-3}$          & 0.16  ${\scriptstyle \pm 2e-2}$          \\
        \textbf{MMtSBM} (ours)                                                        & \textbf{0.019} ${\scriptstyle \pm 4e-4}$ & \textbf{0.130} ${\scriptstyle \pm 2e-3}$  \\
        \bottomrule
    \end{tabular}
\end{table}

\begin{table*}
    \centering
    \fontsize{9}{10}\selectfont
    \caption{\small $\mathcal{W}_1$ of generations vs left-out set for the $d=100$ MULTI benchmark. Reported figures are the average between left-out days $3$ and $4$. Our error margin is over 3 training runs. Best value in \textbf{bold}, second best \underline{underlined}. See \ref{sec:multi_annex} for details \& comments.}
    \begin{tabular}[t]{lll}
        \toprule
        Method                  &                                                                      & $\mathcal{W}_1$ ($\downarrow$)      \\
        \cmidrule(lr){1-3}
        \multicolumn{3}{c}{Schrödinger Bridge}                                                                                               \\
        WLF-SB                  & \citep{neklyudov2024computationalframeworksolvingwasserstein}  & $55.065 {\scriptstyle \,\pm 5.499}$ \\
        {[SF]}$^2$M-Exact       & \citep{tong2024simulationfreeschrodingerbridgesscore}              & $52.888 {\scriptstyle \,\pm 1.986}$ \\
        {[SF]}$^2$M-Geo         & \citep{tong2024simulationfreeschrodingerbridgesscore}                & $52.203 {\scriptstyle \,\pm 1.957}$ \\
        \underline{MMtSBM}      & \underline{(ours)}                                                   & $\underline{44.542 {\scriptstyle \,\pm 0.637}}$ \\
        \cmidrule(lr){1-3}
        \multicolumn{3}{c}{No precomputed OT conditioning}                                                               \\
        I-CFM                   & \citep{tong2024improvinggeneralizingflowbasedgenerative} & $57.262 {\scriptstyle \,\pm 3.855}$ \\
        I-MFM$_{\textrm{RBF}}$  & \citep{kapuśniak2024metricflowmatchingsmooth}                  & $54.197 {\scriptstyle \,\pm 1.408}$ \\
        WLF-UOT                 & \citep{neklyudov2024computationalframeworksolvingwasserstein}   & $54.222 {\scriptstyle \,\pm 5.827}$ \\
        I-CDC-FM                & \citep{bamberger2026carreduchampflow}                           & $54.419 {\scriptstyle \,\pm 0.629}$ \\
        \underline{MMtSBM}      & \underline{(ours)}                                        & $\underline{44.542 {\scriptstyle \,\pm 0.637}}$ \\
    \end{tabular}%
    \begin{tabular}[t]{lll}
        \toprule
        Method                  &                                                                      & $\mathcal{W}_1$ ($\downarrow$)      \\
        \midrule
        \multicolumn{3}{c}{Wasserstein Gradient Flows (WGF)}                                                                           \\
        WLF-SB                  & \citep{neklyudov2024computationalframeworksolvingwasserstein}   & $55.065 {\scriptstyle \,\pm 5.499}$ \\
        WLF-OT                  & \citep{neklyudov2024computationalframeworksolvingwasserstein}   & $55.416 {\scriptstyle \,\pm 6.097}$ \\
        WLF-UOT                 & \citep{neklyudov2024computationalframeworksolvingwasserstein}   & $54.222 {\scriptstyle \,\pm 5.827}$ \\
        \midrule
        \multicolumn{3}{c}{WGF with knowledge of the left-out marginal}                                                \\
        WLF-{\scriptsize (OT+potential)}  & \citep{neklyudov2024computationalframeworksolvingwasserstein} & $47.365 {\scriptstyle \,\pm 0.051}$  \\
        WLF-{\scriptsize (UOT+potential)} & \citep{neklyudov2024computationalframeworksolvingwasserstein} & $45.231 {\scriptstyle \,\pm 0.010}$  \\
        \midrule
        \multicolumn{3}{c}{Flow Matching with exact OT conditioning}                                            \\
        OT-CFM                  & \citep{tong2024improvinggeneralizingflowbasedgenerative}  & $54.814 {\scriptstyle \,\pm 5.858}$ \\
        OT-MFM$_{\textrm{RBF}}$ & \citep{kapuśniak2024metricflowmatchingsmooth}                    & $50.906 {\scriptstyle \,\pm 4.627}$ \\
        OT-CDC-FM               & \citep{bamberger2026carreduchampflow}                          & $52.043 {\scriptstyle \,\pm 1.948}$ \\
        \midrule
        \multicolumn{3}{c}{Metric-aware \textit{interpolation} with exact OT conditioning}                                                     \\
        \textbf{GAGA}           & \citep{sun2025geometryawaregenerativeautoencoderswarped}             & $\mathbf{27.04 {\scriptstyle \,\pm 2.95}}$  \\
    \end{tabular}
    \label{tab:multi-benchmark}
\end{table*}

We next evaluate our method on the Embryoid Body (EB) \citep{Moon2019} and MULTI \citep{pmlr-v176-lance22a} benchmarks, two trajectory inference tasks on real single-cell RNA-seq data. We project RNA counts to their first $d=100$ principal components for each of the $N=5$ and $N=4$ marginals, respectively.\footnote{We actually reuse preprocessed data from \citet{tong2020trajectorynetdynamicoptimaltransport} and \citet{tong2024simulationfreeschrodingerbridgesscore}.} We report the Maximum Mean Discrepancy (MMD) and Sliced Wasserstein Distance (SWD) for EB in Table~\ref{tab:embryoid-benchmark} and Table~\ref{tab:embryoid-benchmark-unnorm}, and the Wasserstein-$1$ distance for MULTI in Table~\ref{tab:multi-benchmark}. For the EB benchmark we use two evaluation settings: either we train MMtSBM on all marginals (Table~\ref{tab:embryoid-benchmark}), or we leave out odd-indexed timesteps (Table~\ref{tab:embryoid-benchmark-unnorm}). For the MULTI benchmark we leave-out one of either intermediate times ($t=1$ or $t=2$) during training.

\begin{table}[H]
    \centering
    \small
    \setlength{\tabcolsep}{3pt}
    \caption{\small MMD and SWD of generations vs held-out times for the $d=100$ \textit{unnormalized} EB benchmark. Our generations start from $\mu_{t=0}^\text{test}$. Figures for SBIRR \citep{pmlr-v258-shen25b}, MMFM \citep{rohbeck2025modeling}, DMSB \citep{chen2023deepmomentummultimarginalschrodinger} and 3MSBM \citep{theodoropoulos2025momentummultimarginalschrodingerbridge} are from \citet{theodoropoulos2025momentummultimarginalschrodingerbridge}. Our standard deviations are over 3 runs.}
    \begin{tabular}{lcccc}
         & MMD $t_1$ $\downarrow$ & SWD $t_1$ $\downarrow$ & MMD $t_3$ $\downarrow$ & SWD $t_3$ $\downarrow$ \\
        \midrule
        SBIRR      & $0.71{\scriptstyle\pm0.08}$ & $0.80{\scriptstyle\pm0.06}$ & $0.73{\scriptstyle\pm0.06}$ & $0.91{\scriptstyle\pm0.05}$ \\
        MMFM       & $0.37{\scriptstyle\pm0.02}$ & $0.59{\scriptstyle\pm0.04}$ & $0.35{\scriptstyle\pm0.04}$ & $0.76{\scriptstyle\pm0.04}$ \\
        DMSB       & $0.38{\scriptstyle\pm0.04}$ & $0.58{\scriptstyle\pm0.06}$ & $0.36{\scriptstyle\pm0.07}$ & $0.54{\scriptstyle\pm0.06}$ \\
        3MSBM      & $0.18{\scriptstyle\pm0.01}$ & $0.48{\scriptstyle\pm0.04}$ & $0.14{\scriptstyle\pm0.04}$ & $0.38{\scriptstyle\pm0.03}$ \\
        \textbf{MMtSBM} {\scriptsize (ours)} & $\mathbf{0.17}{\scriptstyle\pm0.00}$ & $\mathbf{0.45}{\scriptstyle\pm0.03}$ & $\mathbf{0.06}{\scriptstyle\pm0.00}$ & $\mathbf{0.31}{\scriptstyle\pm0.01}$ \\
        \bottomrule
    \end{tabular}
    \label{tab:embryoid-benchmark-unnorm}
\end{table}

On the EB benchmark, our method consistently outperforms baselines on all settings and marginals. It significantly reduces the average MMD by \textbf{-}$\mathbf{41}$\textbf{\%} and the SWD by \textbf{-}$\mathbf{19}$\textbf{\%} compared to DMSB when training on all marginals (Table~\ref{tab:embryoid-benchmark}), demonstrating superior distribution fitting. It also achieves the best performance on fully held-out times (Table~\ref{tab:embryoid-benchmark-unnorm}), where the interpolative prior also matters.\\
On the MULTI benchmark, we again reach significantly better average $\mathcal{W}_1$ distances than the directly comparable literature\footnote{Comparable literature: mainly methods computing the Schrödinger Bridge. Also: methods performing trajectory \emph{inference}, instead of \textit{interpolation}, and those not needing a precomputed OT plan –our goal is to learn a generalizable entropic one!}, beating the previous directly comparable state-of-the-art ({[SF]}$^2$M-Geo) by \textbf{-}$\mathbf{15}$\textbf{\%} with a high statistical significance, and even outperforming methods leveraging precomputed OT plans and specialized to the transcriptomic setting. This demonstrates the applicability of MMtSBM on pure cellular trajectory inference, despite the absence of restrictive modeling such as spline-valued trajectories, explicitly precomputed OT plan, ad-hoc particularizations, or start \textit{and} end true points trajectory pinning.

\subsection{MMtSBM recovers continuous video dynamics from unpaired data}
We now evaluate our method on image-space datasets, where the goal is to recover continuous trajectories (\textit{ie} \emph{videos}) from completely unpaired temporal snapshots.

\subsubsection{MNIST digit morphing}\label{sec:mnist}
We conducted experiments on the MNIST dataset of hand-written digits, transporting digits in decreasing order: $4~\to~3 \to 2 \to 1 \to 0$. The algorithm was trained directly in image space, in dimension $28\times 28 = 784$. As shown in Figure \ref{fig:mnist}, MMtSBM exhibits clear digit morphing, sometimes reusing pixel structures (e.g., the top of the $3$ to form the top of the $2$), which is what is expected from OT in pixel space. This experiment thus demonstrates that MMtSBM manages to learn a complex temporal OT map in image space directly.

\subsubsection{"Biotine" cell culture}\label{sec:biotine}
The (in-house) "biotine" dataset consists of 3-channel $128~\times~128$ fluorescence images (GFP, membrane, nucleus) of A549 lung epithelial cultured cells, treated with biotin, and imaged over 90 minutes at 7 discrete time steps. 

Figure \ref{fig:biotine_true} shows the ground truth unpaired dynamic. We can clearly observe fluorescence loss in the cytoplasmic area, corresponding to the green channel. Interestingly, in the generated video (Figure~\ref{fig:biotine_gen}), contrary to the above MNIST experiment, a mostly static \textit{positional} evolution is observed.
\begin{figure}[h]
    \centering
    \includegraphics[trim=2cm 2cm 2cm 2.4cm, clip, width=0.19\linewidth]{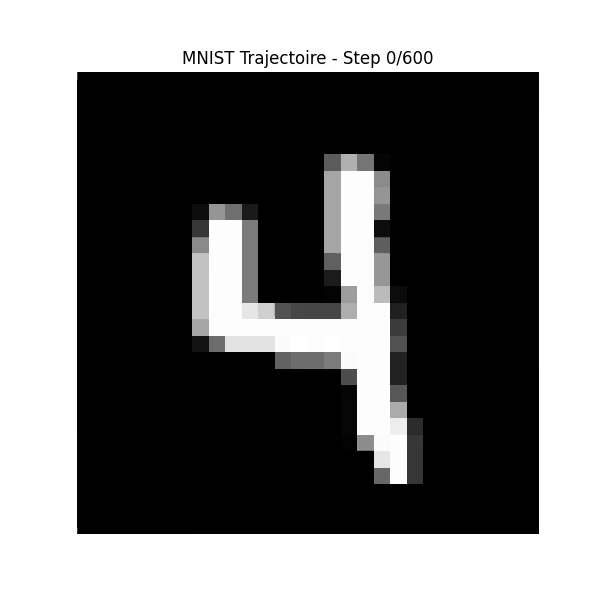}
    \includegraphics[trim=2cm 2cm 2cm 2.4cm, clip, width=0.19\linewidth]{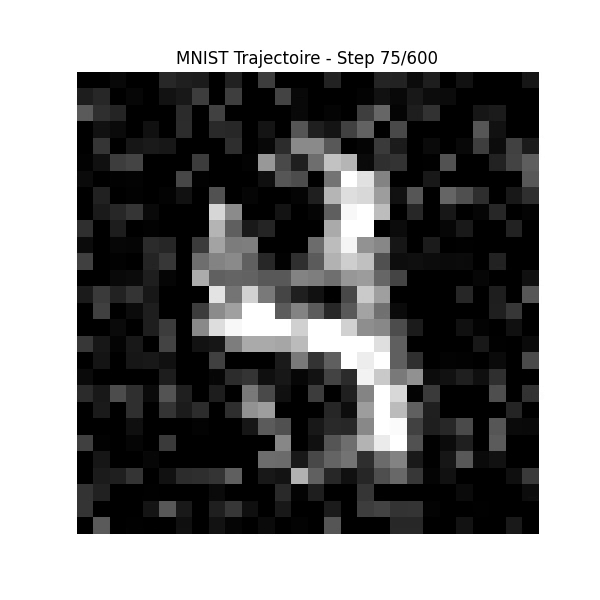}
    \includegraphics[trim=2cm 2cm 2cm 2.4cm, clip, width=0.19\linewidth]{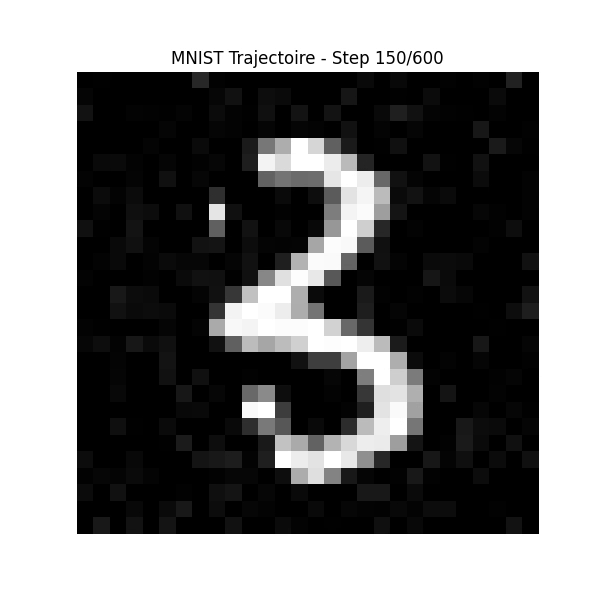}
    \includegraphics[trim=2cm 2cm 2cm 2.4cm, clip, width=0.19\linewidth]{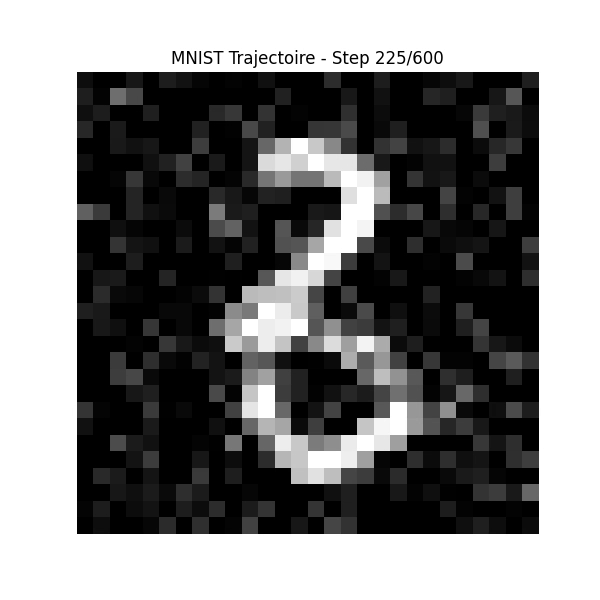}
    \includegraphics[trim=2cm 2cm 2cm 2.4cm, clip, width=0.19\linewidth]{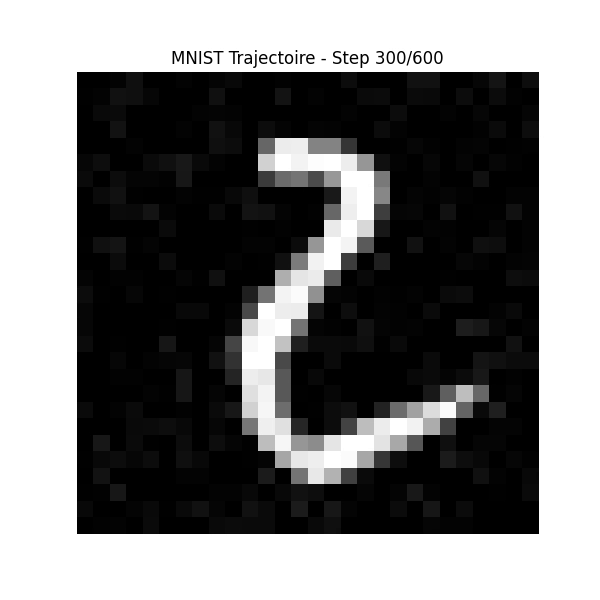}
    \includegraphics[trim=2cm 2cm 2cm 2.4cm, clip, width=0.19\linewidth]{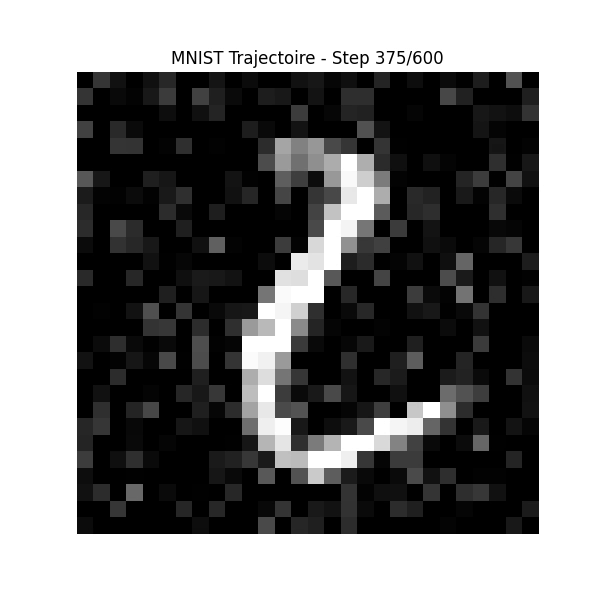}
    \includegraphics[trim=2cm 2cm 2cm 2.4cm, clip, width=0.19\linewidth]{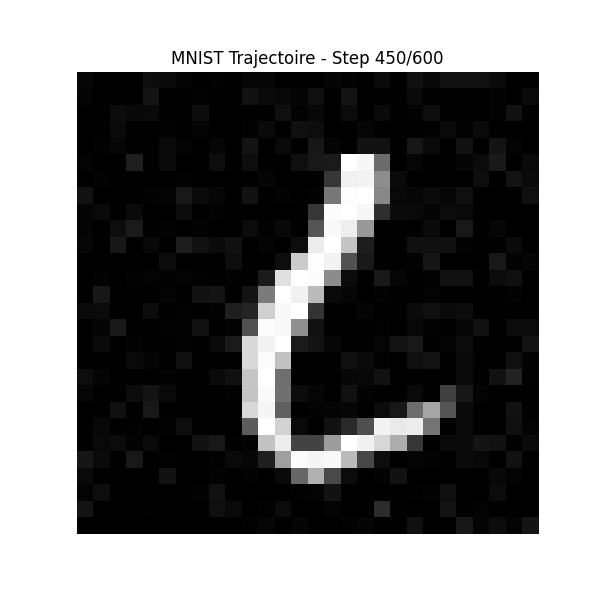}
    \includegraphics[trim=2cm 2cm 2cm 2.4cm, clip, width=0.19\linewidth]{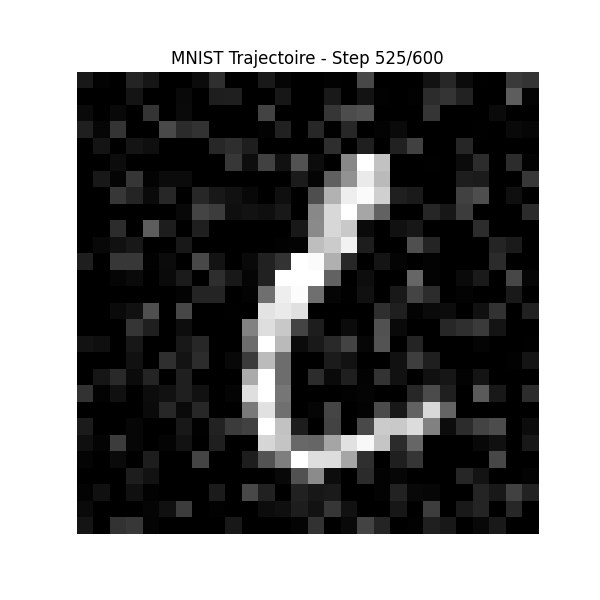}
    \includegraphics[trim=2cm 2cm 2cm 2.4cm, clip, width=0.19\linewidth]{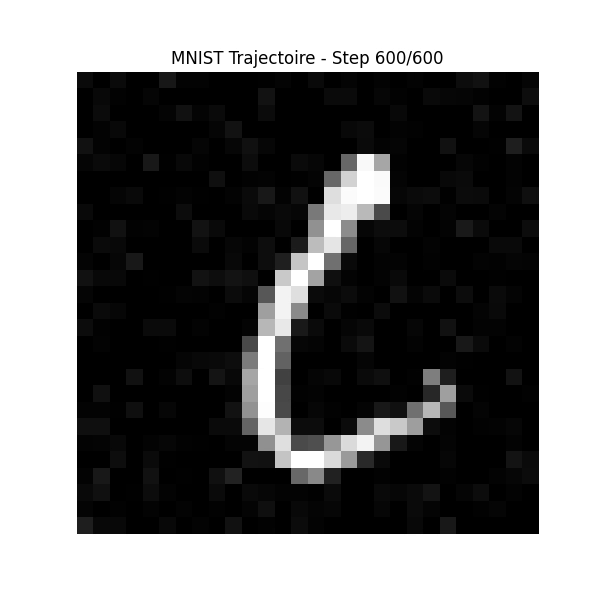}
    \caption{\small Video generated by MMtSBM on MNIST, backward direction. Starting image is from the test set. Left-to-right, top-to-bottom order: generation at time $t=4, 3.5, 3, 2.5, 2, 1.5, 1, 0.5, 0$. Integer times are marginal times.}
    \label{fig:mnist}
\end{figure}
\begin{figure}[h]
    \centering
    \includegraphics[width=0.23\linewidth]{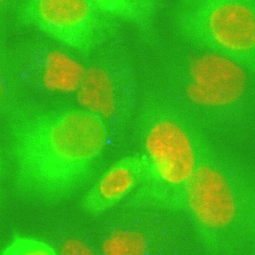}\hspace{0.5mm}
    \includegraphics[width=0.23\linewidth]{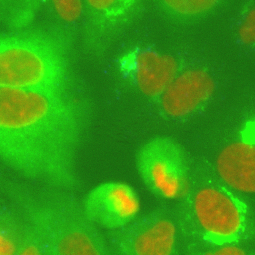}\hspace{0.5mm}
    \includegraphics[width=0.23\linewidth]{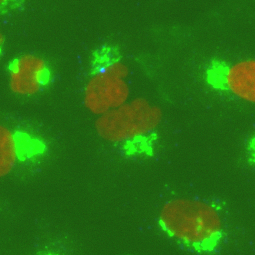}\hspace{0.5mm}
    \includegraphics[width=0.23\linewidth]{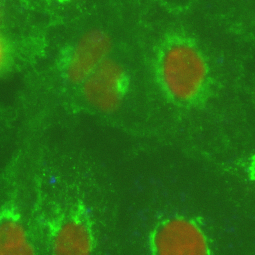}\hspace{0.5mm}
    \includegraphics[width=0.23\linewidth]{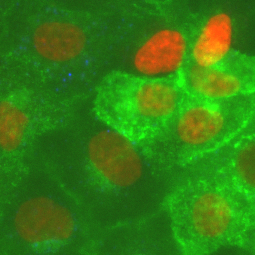}\hspace{0.5mm}
    \includegraphics[width=0.23\linewidth]{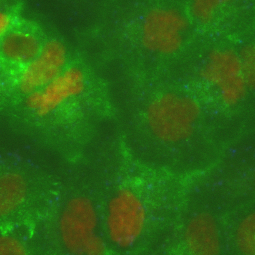}\hspace{0.5mm}
    \includegraphics[width=0.23\linewidth]{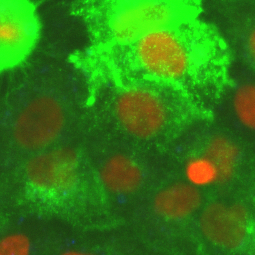}    
    \caption{\small Ground truth biotine examples at training marginal times $t=0,1,2,3,4,5,6$. Left-to-right, top-to-bottom order.}
    \label{fig:biotine_true}
\end{figure}
\begin{figure}[h]
    \centering
    \includegraphics[width=0.24\linewidth]{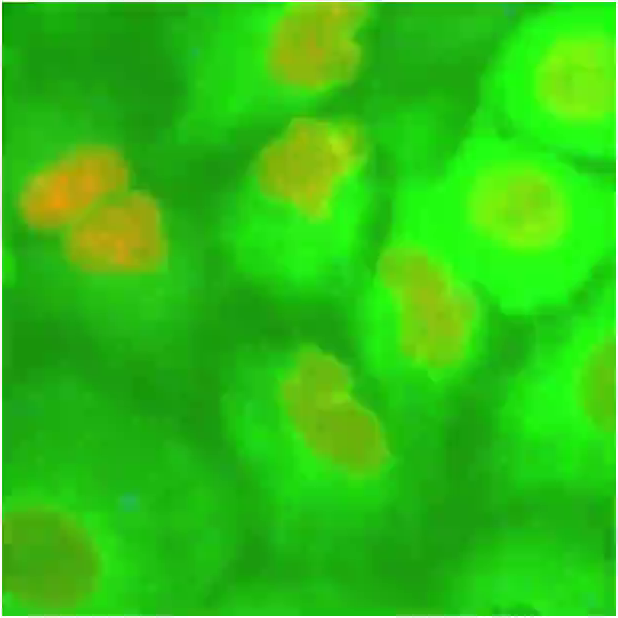}
    \includegraphics[width=0.24\linewidth]{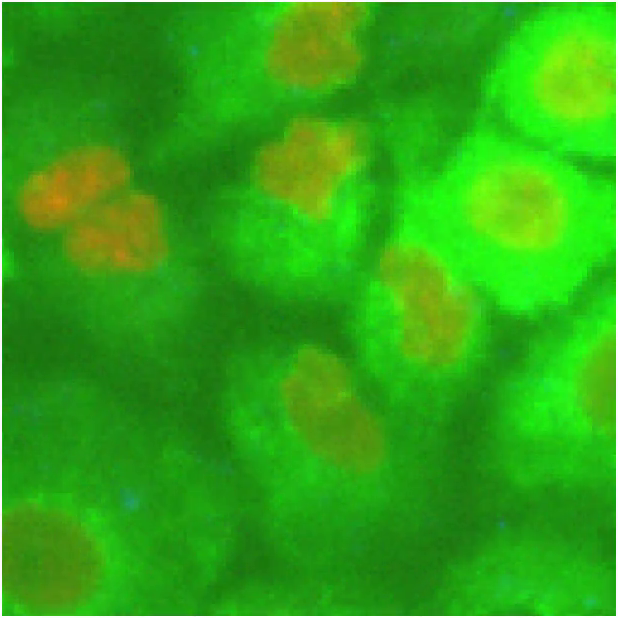}
    \includegraphics[width=0.24\linewidth]{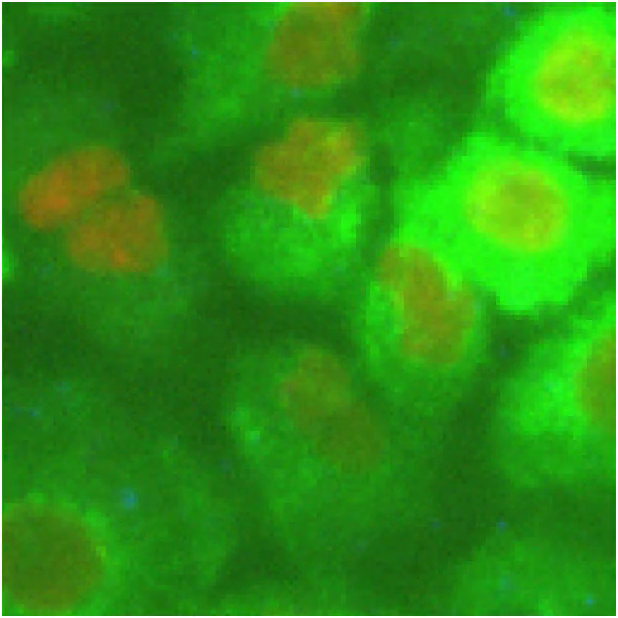}
    \includegraphics[width=0.24\linewidth]{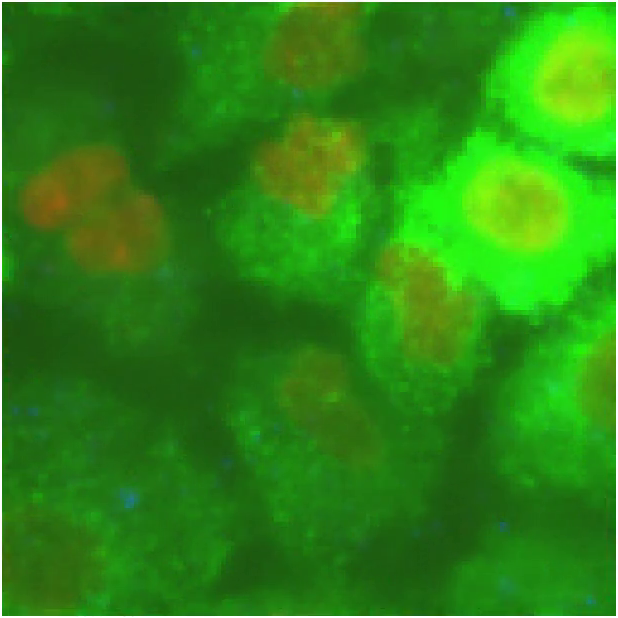}
    \includegraphics[width=0.24\linewidth]{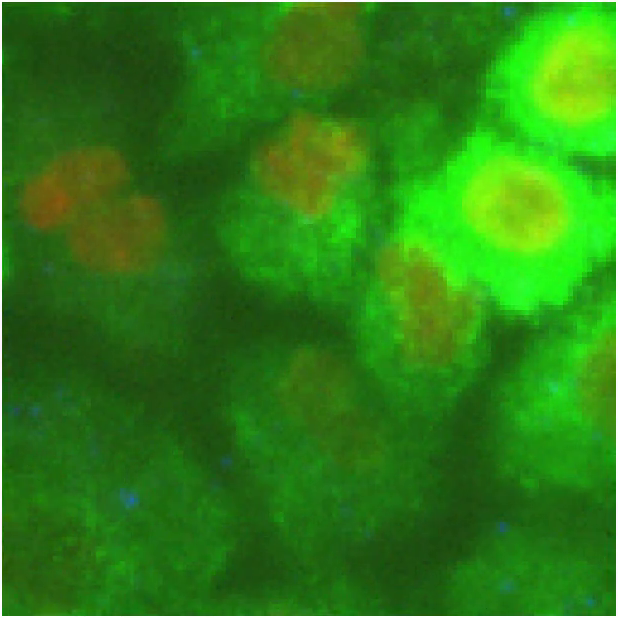}
    \includegraphics[width=0.24\linewidth]{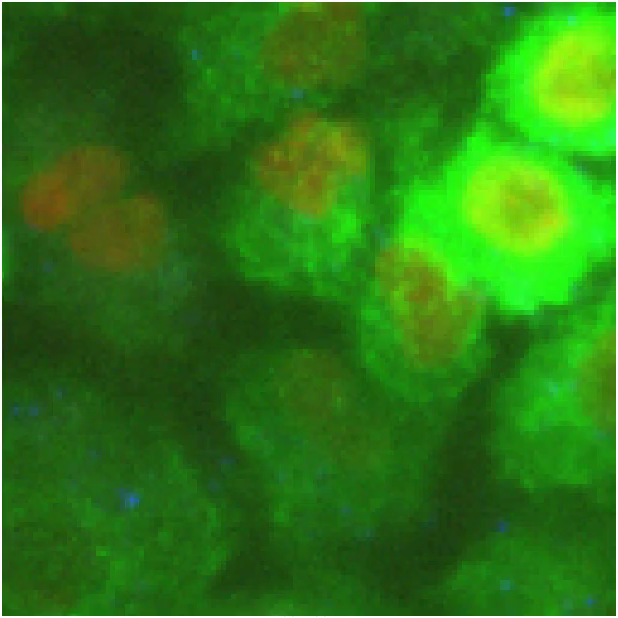}
    \includegraphics[width=0.24\linewidth]{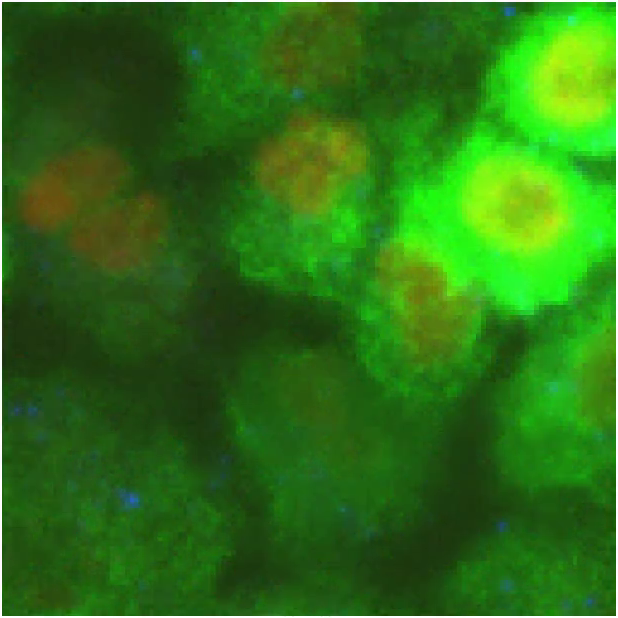}
    \caption{\small Video generated by MMtSBM on biotine, forward direction. Top-left true starting image from the test set, then left-to-right, top-to-bottom order: generations every $\Delta t=1$.}
    \label{fig:biotine_gen}
\end{figure}

This stems from the fact that cell position is not a statistically varying information on the biotine dataset, and this (nonexistent) signal is thus simply not seen by our purely unpaired method, resulting in non-moving cells. MMtSBM rather reconstructs the OT trajectory in pixel space, yielding very close cellular position while still accurately matching the time-varying phenotype (mainly: the fluorescence loss in the cytoplasm). As a baseline reference for future works, we report in Table \ref{tab:kid_metrics} the KID \citep{bińkowski2021demystifyingmmdgans} values of true vs generated samples, using DINOv2 \citep{oquab2024dinov2learningrobustvisual} as the feature extractor.

\begin{table}
    \centering
    \fontsize{8}{9}\selectfont
    \setlength{\tabcolsep}{2pt}
    \caption{\small \textrm{dinov2-vit-b-14}-KIDs for each time, and all times at once.}
    \begin{tabular}{lccccccc}
        \toprule
        Time                    & $t=1$      & $t=2$      & $t=3$      & $t=4$      & $t=5$      & $t=6$      & all times  \\
        KID mean ($\downarrow$) & $11.1 $    & $13.0$     & $20.1$     & $23.5$     & $26.0 $    & $27.7$     & $17.1$     \\
        KID std                 & $\pm 0.23$ & $\pm 0.20$ & $\pm 0.23$ & $\pm 0.25$ & $\pm 0.29$ & $\pm 0.31$ & $\pm 0.32$ \\
        \bottomrule
    \end{tabular}
    \label{tab:kid_metrics}
\end{table}

\subsubsection{KTH Actions}\label{sec:kth}
KTH Actions \citep{kth_actions} is a $160 \times 120$ small video dataset of human actions with varying subjects and conditions. We selected the "running", "jogging" and "walking" categories, and reprocessed the dataset for globally coherent right-to-left motion by flipping the originally left-to-right videos, using $N=5$ marginals.

\begin{figure}[h]
    \centering
    \includegraphics[width=0.32\linewidth]{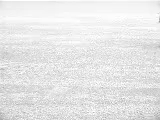}
    \includegraphics[width=0.32\linewidth]{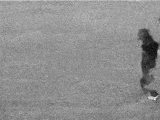}
    \includegraphics[width=0.32\linewidth]{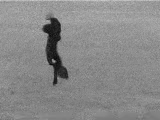}
    \includegraphics[width=0.32\linewidth]{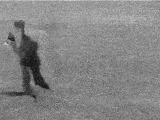}
    \includegraphics[width=0.32\linewidth]{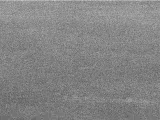}
    \caption{\small Video generated by MMtSBM on KTH Actions, forward direction. Top-left true starting image from the test set, then left-to-right, top-to-bottom order: generations every $\Delta t=1$.}
    \label{fig:kth_gen}
\end{figure}

As can be seen in Figure~\ref{fig:kth_gen}, MMtSBM clearly learns the global right-to-left motion, demonstrating the capability of our method to generate coherent non-static subject movements from purely unpaired data. Although the generation quality appears limited, we hypothesize that it is due to limited data size ($1200$ videos total) and compute budget.

To the best of our knowledge, this is the first demonstration of any method performing video generation from purely unpaired data. Together, this provides evidence for both the scalability to very high-dimensional data and for the fidelity to the underlying biological process of MMtSBM.

\section{Discussion}
In this work we introduce MMtSBM, a novel method that solves the multi-marginal temporal Schrödinger Bridge problem, adapting Bridge Matching \citep{shi2023diffusionschrodingerbridgematching} to our setting. We demonstrate the theoretical soundness of both our modeling and algorithm. We show that MMtSBM indeed produces transport maps that are close to the true OT plan in toy experiments and verify its correct behavior in low-dim experiments. We achieve state-of-the-art results in $2$ widely reported single-cell transcriptomic benchmarks, and for the first time demonstrate a method producing temporally coherent videos from purely unpaired data, hoping to lead to many future scientific applications.

Our method has a number of limitations. Compared to competing methods \citep{tong2024simulationfreeschrodingerbridgesscore, pmlr-v235-gushchin24a}, MMtSBM requires simulating full trajectories during training, which is computationally costly. IMFF is also an iterative optimization algorithm that operates on two distinct networks, resulting in potentially unstable training dynamics. It also importantly necessitates fine-tuning the noise scale hyperparameter, for which we have not designed a principled law. Finally, MMtSBM learns a rather naive dynamic induced by the Brownian interpolant. While it appears to yield excellent distribution matching, it also generates noisy trajectories directly in data space. 

In future works we would like to investigate other regularizations, such as lifting the process to acceleration space to obtain smoother interpolation trajectories, or exploring other empirical reference processes than the Brownian motion. We also intend to investigate learning the transport map in a latent space. We would also like to explore using the single network theory developed in \citet{debortoli2024schrodingerbridgeflowunpaired} for efficiency gains, as well as simulation-free methods.

\newpage 
\section*{Impact Statement}

This paper presents work whose goal is to advance the field of machine learning. There are many potential societal consequences of our work, none of which we feel must be specifically highlighted here.

\bibliography{bib}

@article{schrodinger1931,
  author = {Schr{\"o}dinger, Erwin},
  title = {Über die Umkehrung der Naturgesetze},
  journal = {Sitzungsberichte der Preussischen Akademie der Wissenschaften, Physikalisch-mathematische Klasse},
  year = {1931},
  pages = {144--153}
}

@inproceedings{cuturi2013sinkhorn,
  author = {Cuturi, Marco},
  title = {Sinkhorn Distances: Lightspeed Computation of Optimal Transport},
  booktitle = {Advances in Neural Information Processing Systems (NeurIPS)},
  year = {2013}
}

@article{bunne2022,
  author = {Bunne, Charlotte and Stark, Stefan G. and Gut, Gabriele and Stelzer, Eran H.K. and Rätsch, Gunnar and Cuturi, Marco},
  title = {Learning Single-Cell Perturbation Responses Using Neural Optimal Transport},
  journal = {Nature Methods},
  year = {2023},
  volume = {20},
  number = {12},
  pages = {1820--1829}
}

@inproceedings{song2021,
  author = {Song, Yang and Sohl-Dickstein, Jascha and Kingma, Diederik P. and Kumar, Abhishek and Ermon, Stefano and Poole, Ben},
  title = {Score-Based Generative Modeling through Stochastic Differential Equations},
  booktitle = {International Conference on Learning Representations (ICLR)},
  year = {2021}
}

@article{kullback,
author = {S. Kullback},
title = {{Probability Densities with Given Marginals}},
volume = {39},
journal = {The Annals of Mathematical Statistics},
number = {4},
publisher = {Institute of Mathematical Statistics},
pages = {1236 -- 1243},
year = {1968},
doi = {10.1214/aoms/1177698249},
URL = {https://doi.org/10.1214/aoms/1177698249}
}

@misc{debortoli2021diffusionschrodingerbridgeapplications,
  title={Diffusion Schr\"odinger Bridge with Applications to Score-Based Generative Modeling}, 
  author={De Bortoli, Valentin and Thornton, James and Heng, Jeremy and Doucet, Arnaud},
  year={2021},
  eprint={2106.01357},
  archivePrefix={arXiv},
  primaryClass={stat.ML},
  url={https://arxiv.org/abs/2106.01357}
}

@misc{shi2023diffusionschrodingerbridgematching,
  title={Diffusion Schr\"odinger Bridge Matching}, 
  author={Shi, Yuyang and De Bortoli, Valentin and Campbell, Andrew and Doucet, Arnaud},
  year={2023},
  eprint={2303.16852},
  archivePrefix={arXiv},
  primaryClass={stat.ML},
  url={https://arxiv.org/abs/2303.16852}
}

@article{Peluchetti2023DiffusionBM,
  title={Diffusion Bridge Mixture Transports, Schr{\"o}dinger Bridge Problems and Generative Modeling},
  author={Stefano Peluchetti},
  journal={J. Mach. Learn. Res.},
  year={2023},
  volume={24},
  pages={374:1-374:51},
  url={https://api.semanticscholar.org/CorpusID:257912618}
}

@article{leonard2014survey,
  author = {L{\'e}onard, Christian},
  title = {A survey of the Schr{\"o}dinger problem and some of its connections with optimal transport},
  journal = {Discrete and Continuous Dynamical Systems - A},
  year = {2014},
  volume = {34},
  number = {4},
  pages = {1533--1574}
}

@article{Moon2019,
  author    = {Kevin R. Moon and David van Dijk and Zheng Wang and Scott Gigante and Daniel B. Burkhardt and William S. Chen and Kristina Yim and Antonia van den Elzen and Matthew J. Hirn and Ronald R. Coifman and Natalia B. Ivanova and Guy Wolf and Smita Krishnaswamy},
  title     = {Visualizing structure and transitions in high-dimensional biological data},
  journal   = {Nature Biotechnology},
  year      = {2019},
  volume    = {37},
  number    = {12},
  pages     = {1482--1492},
  doi       = {10.1038/s41587-019-0336-3},
  url       = {https://doi.org/10.1038/s41587-019-0336-3},
  issn      = {1546-1696}
}

@misc{wang2021deepgenerativelearningschrodinger,
      title={Deep Generative Learning via Schr\"{o}dinger Bridge}, 
      author={Gefei Wang and Yuling Jiao and Qian Xu and Yang Wang and Can Yang},
      year={2021},
      eprint={2106.10410},
      archivePrefix={arXiv},
      primaryClass={cs.LG},
      url={https://arxiv.org/abs/2106.10410}, 
}

@article{gyongy1986,
  author = {Gy{\"o}ngy, Istv{\'a}n},
  title = {Mimicking the one-dimensional marginal distributions of processes having an It{\^o} differential},
  journal = {Probability Theory and Related Fields},
  year = {1986},
  volume = {71},
  pages = {501--516}
}

@misc{chen2023likelihoodtrainingschrodingerbridge,
      title={Likelihood Training of Schr\"odinger Bridge using Forward-Backward SDEs Theory}, 
      author={Tianrong Chen and Guan-Horng Liu and Evangelos A. Theodorou},
      year={2023},
      eprint={2110.11291},
      archivePrefix={arXiv},
      primaryClass={stat.ML},
      url={https://arxiv.org/abs/2110.11291}, 
}

@misc{tong2024simulationfreeschrodingerbridgesscore,
      title={Simulation-free Schr\"odinger bridges via score and flow matching}, 
      author={Alexander Tong and Nikolay Malkin and Kilian Fatras and Lazar Atanackovic and Yanlei Zhang and Guillaume Huguet and Guy Wolf and Yoshua Bengio},
      year={2024},
      eprint={2307.03672},
      archivePrefix={arXiv},
      primaryClass={cs.LG},
      url={https://arxiv.org/abs/2307.03672}, 
}

@INPROCEEDINGS{kth_actions,
  author={Schuldt, C. and Laptev, I. and Caputo, B.},
  booktitle={Proceedings of the 17th International Conference on Pattern Recognition, 2004. ICPR 2004.}, 
  title={Recognizing human actions: a local SVM approach}, 
  year={2004},
  volume={3},
  number={},
  pages={32-36 Vol.3},
  doi={10.1109/ICPR.2004.1334462}}

@misc{baradat2020minimizingrelativeentropypath,
      title={Minimizing relative entropy of path measures under marginal constraints}, 
      author={Aymeric Baradat and Christian Léonard},
      year={2020},
      eprint={2001.10920},
      archivePrefix={arXiv},
      primaryClass={math.PR},
      url={https://arxiv.org/abs/2001.10920}, 
}

@misc{chen2023deepmomentummultimarginalschrodinger,
  title={Deep Momentum Multi-Marginal Schr\"odinger Bridge}, 
  author={Chen, Tianrong and Liu, Guan-Horng and Tao, Molei and Theodorou, Evangelos A.},
  year={2023},
  eprint={2303.01751},
  archivePrefix={arXiv},
  primaryClass={stat.ML},
  url={https://arxiv.org/abs/2303.01751}
}

@misc{theodoropoulos2025momentummultimarginalschrodingerbridge,
  title={Momentum Multi-Marginal Schr\"odinger Bridge Matching}, 
  author={Theodoropoulos, Panagiotis and Saravanos, Augustinos D. and Theodorou, Evangelos A. and Liu, Guan-Horng},
  year={2025},
  eprint={2506.10168},
  archivePrefix={arXiv},
  primaryClass={stat.ML},
  url={https://arxiv.org/abs/2506.10168}
}

@misc{lipman2023flowmatchinggenerativemodeling,
      title={Flow Matching for Generative Modeling}, 
      author={Yaron Lipman and Ricky T. Q. Chen and Heli Ben-Hamu and Maximilian Nickel and Matt Le},
      year={2023},
      eprint={2210.02747},
      archivePrefix={arXiv},
      primaryClass={cs.LG},
      url={https://arxiv.org/abs/2210.02747}, 
}

@misc{open-problems-multimodal,
    author = {Daniel Burkhardt and Malte Luecken and Andrew Benz and Peter Holderrieth and Jonathan Bloom and Christopher Lance and Ashley Chow and Ryan Holbrook},
    title = {Open Problems - Multimodal Single-Cell Integration},
    year = {2022},
    howpublished = {\url{https://kaggle.com/competitions/open-problems-multimodal}},
    note = {Kaggle}
}

@InProceedings{pmlr-v176-lance22a,
  title = 	 {Multimodal single cell data integration challenge: Results and lessons learned},
  author =       {Lance, Christopher and Luecken, Malte D. and Burkhardt, Daniel B. and Cannoodt, Robrecht and Rautenstrauch, Pia and Laddach, Anna and Ubingazhibov, Aidyn and Cao, Zhi-Jie and Deng, Kaiwen and Khan, Sumeer and Liu, Qiao and Russkikh, Nikolay and Ryazantsev, Gleb and Ohler, Uwe and data integration competition participants, NeurIPS 2021 Multimodal and Pisco, Angela Oliveira and Bloom, Jonathan and Krishnaswamy, Smita and Theis, Fabian J.},
  booktitle = 	 {Proceedings of the NeurIPS 2021 Competitions and Demonstrations Track},
  pages = 	 {162--176},
  year = 	 {2022},
  editor = 	 {Kiela, Douwe and Ciccone, Marco and Caputo, Barbara},
  volume = 	 {176},
  series = 	 {Proceedings of Machine Learning Research},
  month = 	 {06--14 Dec},
  publisher =    {PMLR},
  url = 	 {https://proceedings.mlr.press/v176/lance22a.html},
}

@misc{kingma2017adammethodstochasticoptimization,
      title={Adam: A Method for Stochastic Optimization}, 
      author={Diederik P. Kingma and Jimmy Ba},
      year={2017},
      eprint={1412.6980},
      archivePrefix={arXiv},
      primaryClass={cs.LG},
      url={https://arxiv.org/abs/1412.6980}, 
}

@misc{loshchilov2019decoupledweightdecayregularization,
      title={Decoupled Weight Decay Regularization}, 
      author={Ilya Loshchilov and Frank Hutter},
      year={2019},
      eprint={1711.05101},
      archivePrefix={arXiv},
      primaryClass={cs.LG},
      url={https://arxiv.org/abs/1711.05101}, 
}

@misc{sun2025geometryawaregenerativeautoencoderswarped,
      title={Geometry-Aware Generative Autoencoders for Warped Riemannian Metric Learning and Generative Modeling on Data Manifolds}, 
      author={Xingzhi Sun and Danqi Liao and Kincaid MacDonald and Yanlei Zhang and Chen Liu and Guillaume Huguet and Guy Wolf and Ian Adelstein and Tim G. J. Rudner and Smita Krishnaswamy},
      year={2025},
      eprint={2410.12779},
      archivePrefix={arXiv},
      primaryClass={cs.LG},
      url={https://arxiv.org/abs/2410.12779}, 
}

@misc{park2025multimarginalschrodingerbridgematching,
      title={Multi-Marginal Schr\"odinger Bridge Matching}, 
      author={Byoungwoo Park and Juho Lee},
      year={2025},
      eprint={2510.16587},
      archivePrefix={arXiv},
      primaryClass={stat.ML},
      url={https://arxiv.org/abs/2510.16587}, 
}

@misc{bińkowski2021demystifyingmmdgans,
      title={Demystifying MMD GANs}, 
      author={Mikołaj Bińkowski and Danica J. Sutherland and Michael Arbel and Arthur Gretton},
      year={2021},
      eprint={1801.01401},
      archivePrefix={arXiv},
      primaryClass={stat.ML},
      url={https://arxiv.org/abs/1801.01401}, 
}

@misc{oquab2024dinov2learningrobustvisual,
      title={DINOv2: Learning Robust Visual Features without Supervision}, 
      author={Maxime Oquab and Timothée Darcet and Théo Moutakanni and Huy Vo and Marc Szafraniec and Vasil Khalidov and Pierre Fernandez and Daniel Haziza and Francisco Massa and Alaaeldin El-Nouby and Mahmoud Assran and Nicolas Ballas and Wojciech Galuba and Russell Howes and Po-Yao Huang and Shang-Wen Li and Ishan Misra and Michael Rabbat and Vasu Sharma and Gabriel Synnaeve and Hu Xu and Hervé Jegou and Julien Mairal and Patrick Labatut and Armand Joulin and Piotr Bojanowski},
      year={2024},
      eprint={2304.07193},
      archivePrefix={arXiv},
      primaryClass={cs.CV},
      url={https://arxiv.org/abs/2304.07193}, 
}

@misc{neklyudov2024computationalframeworksolvingwasserstein,
      title={A Computational Framework for Solving Wasserstein Lagrangian Flows}, 
      author={Kirill Neklyudov and Rob Brekelmans and Alexander Tong and Lazar Atanackovic and Qiang Liu and Alireza Makhzani},
      year={2024},
      eprint={2310.10649},
      archivePrefix={arXiv},
      primaryClass={cs.LG},
      url={https://arxiv.org/abs/2310.10649}, 
}

@misc{huguet2022manifoldinterpolatingoptimaltransportflows,
      title={Manifold Interpolating Optimal-Transport Flows for Trajectory Inference}, 
      author={Guillaume Huguet and D. S. Magruder and Alexander Tong and Oluwadamilola Fasina and Manik Kuchroo and Guy Wolf and Smita Krishnaswamy},
      year={2022},
      eprint={2206.14928},
      archivePrefix={arXiv},
      primaryClass={cs.LG},
      url={https://arxiv.org/abs/2206.14928}, 
}

@misc{albergo2023buildingnormalizingflowsstochastic,
      title={Building Normalizing Flows with Stochastic Interpolants}, 
      author={Michael S. Albergo and Eric Vanden-Eijnden},
      year={2023},
      eprint={2209.15571},
      archivePrefix={arXiv},
      primaryClass={cs.LG},
      url={https://arxiv.org/abs/2209.15571}, 
}

@misc{liu2022flowstraightfastlearning,
      title={Flow Straight and Fast: Learning to Generate and Transfer Data with Rectified Flow}, 
      author={Xingchao Liu and Chengyue Gong and Qiang Liu},
      year={2022},
      eprint={2209.03003},
      archivePrefix={arXiv},
      primaryClass={cs.LG},
      url={https://arxiv.org/abs/2209.03003}, 
}

@misc{bamberger2026carreduchampflow,
      title={Carr\'e du champ flow matching: better quality-generalisation tradeoff in generative models}, 
      author={Jacob Bamberger and Iolo Jones and Dennis Duncan and Michael M. Bronstein and Pierre Vandergheynst and Adam Gosztolai},
      year={2026},
      eprint={2510.05930},
      archivePrefix={arXiv},
      primaryClass={cs.LG},
      url={https://arxiv.org/abs/2510.05930}, 
}

@InProceedings{pmlr-v258-shen25b,
  title = 	 {Multi-marginal Schr{ö}dinger Bridges with Iterative Reference Refinement},
  author =       {Shen, Yunyi and Berlinghieri, Renato and Broderick, Tamara},
  booktitle = 	 {Proceedings of The 28th International Conference on Artificial Intelligence and Statistics},
  pages = 	 {3817--3825},
  year = 	 {2025},
  editor = 	 {Li, Yingzhen and Mandt, Stephan and Agrawal, Shipra and Khan, Emtiyaz},
  volume = 	 {258},
  series = 	 {Proceedings of Machine Learning Research},
  month = 	 {03--05 May},
  publisher =    {PMLR},
  url = 	 {https://proceedings.mlr.press/v258/shen25b.html},
}

@inproceedings{rohbeck2025modeling,
    title={Modeling Complex System Dynamics with Flow Matching Across Time and Conditions},
    author={Martin Rohbeck and Charlotte Bunne and Edward De Brouwer and Jan-Christian Huetter and Anne Biton and Kelvin Y. Chen and Aviv Regev and Romain Lopez},
    booktitle={The Thirteenth International Conference on Learning Representations},
    year={2025},
    url={https://openreview.net/forum?id=hwnObmOTrV}
}

@InProceedings{pmlr-v235-gushchin24a,
  title = 	 {Light and Optimal Schrödinger Bridge Matching},
  author =       {Gushchin, Nikita and Kholkin, Sergei and Burnaev, Evgeny and Korotin, Alexander},
  booktitle = 	 {Proceedings of the 41st International Conference on Machine Learning},
  pages = 	 {17100--17122},
  year = 	 {2024},
  editor = 	 {Salakhutdinov, Ruslan and Kolter, Zico and Heller, Katherine and Weller, Adrian and Oliver, Nuria and Scarlett, Jonathan and Berkenkamp, Felix},
  volume = 	 {235},
  series = 	 {Proceedings of Machine Learning Research},
  month = 	 {21--27 Jul},
  publisher =    {PMLR},
  url = 	 {https://proceedings.mlr.press/v235/gushchin24a.html},
}

@misc{kapuśniak2024metricflowmatchingsmooth,
      title={Metric Flow Matching for Smooth Interpolations on the Data Manifold}, 
      author={Kacper Kapuśniak and Peter Potaptchik and Teodora Reu and Leo Zhang and Alexander Tong and Michael Bronstein and Avishek Joey Bose and Francesco Di Giovanni},
      year={2024},
      eprint={2405.14780},
      archivePrefix={arXiv},
      primaryClass={cs.LG},
      url={https://arxiv.org/abs/2405.14780}, 
}

@misc{koshizuka2023neurallagrangianschrodingerbridge,
      title={Neural Lagrangian Schr\"odinger Bridge: Diffusion Modeling for Population Dynamics}, 
      author={Takeshi Koshizuka and Issei Sato},
      year={2023},
      eprint={2204.04853},
      archivePrefix={arXiv},
      primaryClass={cs.LG},
      url={https://arxiv.org/abs/2204.04853}, 
}

@InProceedings{mmsb,
author="Chen, Yongxin
and Conforti, Giovanni
and Georgiou, Tryphon T.
and Ripani, Luigia",
editor="Nielsen, Frank
and Barbaresco, Fr{\'e}d{\'e}ric",
title="Multi-marginal Schr{\"o}dinger Bridges",
booktitle="Geometric Science of Information",
year="2019",
publisher="Springer International Publishing",
address="Cham",
pages="725--732",
isbn="978-3-030-26980-7"
}

@article{lavenant,
author = {Hugo Lavenant and Stephen Zhang and Young-Heon Kim and Geoffrey Schiebinger},
title = {{Toward a mathematical theory of trajectory inference}},
volume = {34},
journal = {The Annals of Applied Probability},
number = {1A},
publisher = {Institute of Mathematical Statistics},
pages = {428 -- 500},
year = {2024},
doi = {10.1214/23-AAP1969},
URL = {https://doi.org/10.1214/23-AAP1969}
}

@misc{hong2025trajectoryinferencesmoothschrodinger,
  title={Trajectory Inference with Smooth Schr\"odinger Bridges}, 
  author={Hong, Wanli and Shi, Yuliang and Niles-Weed, Jonathan},
  year={2025},
  eprint={2503.00530},
  archivePrefix={arXiv},
  primaryClass={stat.ML},
  url={https://arxiv.org/abs/2503.00530}
}

@article{leonard2012,
  author = {L{\'e}onard, Christian},
  title = {From the Schr{\"o}dinger problem to the Monge--Kantorovich problem},
  journal = {Journal of Functional Analysis},
  year = {2012},
  volume = {262},
  number = {4},
  pages = {1879--1920}
}

@article{Csiszar1975,
  author    = {Imre Csisz{\'a}r},
  title     = {I-Divergence Geometry of Probability Distributions and Minimization Problems},
  journal   = {The Annals of Probability},
  volume    = {3},
  number    = {1},
  pages     = {146--158},
  year      = {1975},
  publisher = {Institute of Mathematical Statistics},
  doi       = {10.1214/aop/1176996454},
  url       = {https://doi.org/10.1214/aop/1176996454}
}

@misc{tong2024improvinggeneralizingflowbasedgenerative,
      title={Improving and generalizing flow-based generative models with minibatch optimal transport}, 
      author={Alexander Tong and Kilian Fatras and Nikolay Malkin and Guillaume Huguet and Yanlei Zhang and Jarrid Rector-Brooks and Guy Wolf and Yoshua Bengio},
      year={2024},
      eprint={2302.00482},
      archivePrefix={arXiv},
      primaryClass={cs.LG},
      url={https://arxiv.org/abs/2302.00482}, 
}

@misc{tong2020trajectorynetdynamicoptimaltransport,
      title={TrajectoryNet: A Dynamic Optimal Transport Network for Modeling Cellular Dynamics}, 
      author={Alexander Tong and Jessie Huang and Guy Wolf and David van Dijk and Smita Krishnaswamy},
      year={2020},
      eprint={2002.04461},
      archivePrefix={arXiv},
      primaryClass={stat.ML},
      url={https://arxiv.org/abs/2002.04461}, 
}

@misc{debortoli2024schrodingerbridgeflowunpaired,
      title={Schr\"odinger Bridge Flow for Unpaired Data Translation}, 
      author={Valentin De Bortoli and Iryna Korshunova and Andriy Mnih and Arnaud Doucet},
      year={2024},
      eprint={2409.09347},
      archivePrefix={arXiv},
      primaryClass={cs.LG},
      url={https://arxiv.org/abs/2409.09347}, 
}
\bibliographystyle{icml2026}

\appendix
\onecolumn
\section{Appendix}

\subsection{Additional Background}
\label{app:background}

\paragraph{Reciprocal projection.}
The reciprocal class $\mathcal{R}(\mathbb{Q})$ consists of mixtures of 
$\mathbb{Q}$-bridges. For $\mathbb{P}\in\mathcal{P}(C)$,
\[
\proj_{\mathcal{R}(\mathbb{Q})}(\mathbb{P})
= \mathbb{P}_{0,T}\,\mathbb{Q}_{|0,T}.
\]

\paragraph{Markovian projection.}
The Markov class $\mathcal{M}$ consists of diffusions 
$\mathrm{d}X_t=v(t,X_t)\,\mathrm{d}t+\sigma\,\mathrm{d}B_t$.  
The projection $\proj_\mathcal{M}(\Pi)$ has drift
\[
\mathrm{d}X_t
=\Bigg[\frac{\mathbb{E}_\Pi[X_T \mid X_t]-X_t}{T-t}\Bigg]\mathrm{d}t
+\sigma\,\mathrm{d}B_t.
\]

\paragraph{Variational formulations.}
Both projections solve KL problems:
\[
\proj_{\mathcal{R}(\mathbb{Q})}(\mathbb{P})
= \argmin_{\Pi\in\mathcal{R}(\mathbb{Q})}\mathrm{KL}(\mathbb{P}\,\|\,\Pi),
\qquad
\proj_{\mathcal{M}}(\Pi)
= \argmin_{M\in\mathcal{M}}\mathrm{KL}(\Pi\,\|\,M).
\]

\paragraph{Bridge matching.}
In practice, the Markov drift is learned by minimising
\[
\mathcal{L}(\theta)=\int_0^T 
\mathbb{E}_{(X_0,X_T)\sim \Pi_{0,T},\,
X_t\sim \mathbb{Q}(\cdot|X_0,X_T)}
\bigg[\|v_\theta(X_t,t)-\tfrac{X_T-X_t}{T-t}\|^2\bigg]\,dt.
\]

\paragraph{Iterative Proportional Fitting (IPF).}
IPF alternately enforces marginals by KL minimisation:
\[
\mathbb{P}^{2n+1}=
\argmin_{\mathbb{P}: \mathbb{P}_T=\mu_T}\mathrm{KL}(\mathbb{P}\,\|\,\mathbb{P}^{2n}),
\quad
\mathbb{P}^{2n+2}=
\argmin_{\mathbb{P}: \mathbb{P}_0=\mu_0}\mathrm{KL}(\mathbb{P}\,\|\,\mathbb{P}^{2n+1}).
\]
Unlike IMF, this requires caching full trajectories.

\subsection{Other properties on IMFF or MMSB}

\begin{proposition}[Markov implies reciprocal]\label{prop:markov_implies_reciprocal}
Any Markov measure on $C([0,T],\mathbb R^d)$ is reciprocal. 
Hence $P^\star \in \mathcal R(Q)$. 
See Proposition~2.3 in \citet{leonard2012}.
\end{proposition}

\begin{proposition}[Sampling with ODE probability flow]
Given the forward and backward drifts of the multi-marginal Schrödinger 
bridge, one can simulate trajectories using the probability flow ODE 
(\citep{song2021}):
\[
\frac{dX_t}{dt} = f_t(X_t) - \tfrac{1}{2}\sigma_t^2 \nabla \log p_t(X_t).
\]
Although the score function $\nabla \log p_t$ is not directly available, 
\citep{debortoli2021diffusionschrodingerbridgeapplications} show that it can be equivalently recovered by 
averaging the forward and backward drifts:
\begin{equation}
    v_t(x) \;=\; \tfrac{1}{2}\Big( v^{\mathrm{fwd}}_t(x) + v^{\mathrm{bwd}}_t(x) \Big)
    \label{eq:ODE_drift}
\end{equation}
Simulating the ODE with drift $v_t$ thus yields a deterministic sampling 
procedure that preserves the marginals of the stochastic bridge, providing 
an efficient and numerically stable alternative to direct SDE simulation.
\end{proposition}

\pagebreak
\subsection{Concrete Algorithms}\label{subsec:concrete_algo}

We always start trainings with a warmup phase, akin to \citet{shi2023diffusionschrodingerbridgematching}. It allows MMtSBM to start rectifying the trajectories from a non-random state, which could be complicated because the IMFF phase uses the forward/backward network to train the backward/forward one.

\begin{algorithm}[H]
    \caption{Warmup (our algorithm)}
    \begin{algorithmic}[1]
    
    \Input{Subdivision $\{0 = t_0 < t_1 < \dots < t_n = T\}$, datasets $\{\pi_{t_i}\}$, 
    networks $v_\theta, v_\phi$, initial params $\theta,\phi$, 
    batch size $B$, warmup steps $N_{\text{warmup}}$}
    
    \State Define bridges $\mathcal{B} = \{(t_i,t_{i+1})\}$, $b \gets B/|\mathcal{B}|$
    
    \For{direction $\in \{\text{forward}, \text{backward}\}$}
      \For{$n \in \llbracket 0, N_{\text{warmup}} \rrbracket$}
        \ForAll{$(t_i,t_{i+1}) \in \mathcal{B}$ \textbf{in parallel}}
          \State Sample $(X_{t_i},X_{t_{i+1}}) \sim (\pi_{t_i}\otimes\pi_{t_{i+1}})^{\otimes b}$, 
                 $t_{(i)} \sim \text{Unif}[t_i,t_{i+1}]^{\otimes b}$
        \EndFor
        \State Aggregate $X_{\text{init}}, X_{\text{final}}, t$; Sample $Z \sim \mathcal{N}(0,I)^{\otimes B}$
        \State $X_t \gets \mathrm{Interp}_{t}(X_{\text{init}}, X_{\text{final}}, Z)$ \Comment{cf. \eqref{eq:interp}}
        \State Update $\theta$ if forward with $\ell^{\text{fwd}}$~\eqref{eq:loss_fwd}, else $\phi$ with $\ell^{\text{bwd}}$~\eqref{eq:loss_bwd}
      \EndFor
    \EndFor
    
    \Output{Warmup parameters $\theta,\phi$}
    
    \end{algorithmic}
\end{algorithm}

\begin{algorithm}[H]
    \caption{Iterative Markovian Factorized Fitting (IMFF) (our algorithm)}
    \begin{algorithmic}[1]
    
    \Input{Subdivision $\{0 = t_0 < t_1 < \dots < t_n = T\}$, datasets $\{\pi_{t_i}\}$, 
    networks $v_\theta, v_\phi$, warmup params $\theta,\phi$, 
    batch size $B$, finetune steps $N_{\text{finetune}}$, inner steps $N_{\text{inner}}$}
    
    \State Define bridges $\mathcal{B} = \{(t_i,t_{i+1})\}$, $b \gets B / |\mathcal{B}|$
    
    \For{$N \in \llbracket 0, N_{\text{finetune}} \rrbracket$}
      \ForAll{$(t_i,t_{i+1}) \in \mathcal{B}$ \textbf{in parallel}} 
        \State Sample $(X_{t_i},X_{t_{i+1}})$ from $(\pi_{t_i}\otimes \pi_{t_{i+1}})^{\otimes b}$
        \State Sample $t_{(i)} \sim \text{Unif}[t_i,t_{i+1}]^{\otimes b}$
      \EndFor
      \State Aggregate $X_{\text{init}}, X_{\text{final}}, t$
    
      \For{direction $\in \{\text{backward}, \text{forward}\}$}
        \For{$n \in \llbracket 0, N_{\text{inner}} \rrbracket$}
          \State Sample $Z \sim \mathcal{N}(0,I)^{\otimes B}$
          \If{direction = forward} 
            \State $\hat{X}_{\text{init}} \gets \mathrm{SDE}(X_{\text{final}}, v_\phi)$ \Comment{cf. \eqref{eq:sim_sde}}
            \State $X_t \gets \mathrm{Interp}_{t}(\hat{X}_{\text{init}}, X_{\text{final}}, Z)$ \Comment{cf. \eqref{eq:interp}}
            \State Update $\theta$ with $\ell^{\text{fwd}}$~\eqref{eq:loss_fwd}
          \Else
            \State $\hat{X}_{\text{final}} \gets \mathrm{SDE}(X_{\text{init}}, v_\theta)$ \Comment{cf. \eqref{eq:sim_sde}}
            \State $X_t \gets \mathrm{Interp}_{t}(X_{\text{init}}, \hat{X}_{\text{final}}, Z)$ \Comment{cf. \eqref{eq:interp}}
            \State Update $\phi$ with $\ell^{\text{bwd}}$~\eqref{eq:loss_bwd}
          \EndIf
        \EndFor
      \EndFor
    \EndFor
    
    \Output{Finetuned parameters $\theta,\phi$}
    
    \end{algorithmic}
\end{algorithm}

\newpage
\subsection{Critical Implementation Considerations}\label{sec:crit_implem}
A naive implementation of the algorithm quickly led to the \emph{forgetting} of paths between marginals as training progressed.  
To overcome this, we developed a fully vectorized implementation that ensures stable learning across all intervals.  
This design is essential for the quality of our solution. Key components are detailed below.

\subsubsection{Scalability with High Dimensions and Many Marginals}

Both Markovian and reciprocal projections are implemented in a fully vectorized manner. Instead of looping over intervals, all pairs are aggregated into global vectors and processed simultaneously on GPU.  

At iteration $n$, for interval $[t_i,t_{i+1}]$, pairs are sampled as  
\[
z_i \sim (M^n)_{t_i},\quad z_{i+1}\sim \mu_{i+1}
\quad\text{(forward)},\qquad
z_{i+1}\sim (M^n)_{t_{i+1}},\quad z_i\sim \mu_i
\quad\text{(backward)}.
\]  
Pairs from all intervals form two batched vectors $(Z_{\text{init}},Z_{\text{final}})$. Each bridge is then simulated in parallel as  
\[
X^{(b)}_t = (1-s)\,z^{(b)}_{\text{init}} + s\,z^{(b)}_{\text{final}}
+ \sigma_t \sqrt{s(1-s)}\,\xi^{(b)},\qquad \xi^{(b)}\sim\mathcal{N}(0,I).
\]  
This parallelization makes multi-marginal training feasible at scale.

\subsubsection{Masking and Time Discretization}
The horizon $[0,T]$ is discretized into $N_{\text{total}}$ steps, allocated proportionally to interval length:  
\[
N_i=\Big\lfloor N_{\text{total}}\tfrac{t_{i+1}-t_i}{T}\Big\rfloor,
\quad
dt_i=\pm\tfrac{\Delta\tau}{t_{i+1}-t_i},\;
\Delta\tau=\tfrac{T_{\max}-T_{\min}}{N_{\text{total}}}.
\]  
This ensures consistent integration with bounded cost.

Since $N_i$ varies across intervals, all trajectories are embedded into a common tensor of shape \texttt{(num\_bridges,max\_N)} with binary masks:  
\[
z^{(b)}_{k+1}=z^{(b)}_k+ v(z^{(b)}_k,t^{(b)}_k)\,dt^{(b)}
+\sigma_{t^{(b)}_k}\sqrt{dt^{(b)}}\,\xi^{(b)},
\]
updated only where \texttt{mask}=1. This allows heterogeneous bridges to evolve in a single GPU loop.

\subsubsection{Interpolation Operator and Losses}\label{subsec:interp_losses}

For each bridge $(t_i,t_{i+1})$ and batch $B$, define  
\[
\mathbf{s} = \frac{\mathbf{t}-t_{\text{init}}}{t_{\text{final}}-t_{\text{init}}} \in [0,1]^B.
\]  
Then the interpolation is  
\begin{equation}
\label{eq:interp}
\mathrm{Interp}_{\mathbf{t}}(X_{\text{init}},X_{\text{final}},Z)
= (1-\mathbf{s})\odot X_{\text{init}}
+ \mathbf{s}\odot X_{\text{final}}
+ \sqrt{\varepsilon (1-\mathbf{s})\odot \mathbf{s}} \odot Z,
\end{equation}
with $\odot$ the elementwise product.

We also define a generic simulation operator for SDEs.  
Given an initial condition $X_{\text{init}}$ and a drift $v_{\text{direction}}$ (either forward or backward), we denote
\begin{equation}
\label{eq:sim_sde}
\mathrm{SDE}(X_{\text{init}}, v_{\text{direction}}): 
\quad dX_t = v_{\text{direction}}(t,X_t)\,dt + \sigma_t\, dB_t, 
\qquad X_{t_{\text{init}}} = X_{\text{init}}.
\end{equation}
This operator returns a trajectory $(X_t)_{t\in[t_{\text{init}},t_{\text{final}}]}$.

Forward/backward losses enforce vectorized drift consistency:
\begin{equation}\label{eq:loss_fwd}
\ell^{\text{fwd}}(\theta;\mathbf{t},X_{\text{final}},X_t)
=\tfrac{1}{B}\big\|v_\theta(\mathbf{t},X_t)
-\tfrac{X_{\text{final}}-X_t}{t_{\text{final}}-\mathbf{t}}\big\|^2
\end{equation}
\begin{equation}\label{eq:loss_bwd}
\ell^{\text{bwd}}(\phi;\mathbf{t},X_{\text{init}},X_t)
=\tfrac{1}{B}\big\|v_\phi(\mathbf{t},X_t)
-\tfrac{X_{\text{init}}-X_t}{\mathbf{t}-t_{\text{init}}}\big\|^2
\end{equation}

\subsubsection{Time-Dependent Drift Networks}

The drifts $v_\theta,v_\phi$ are parameterized by networks with explicit time encodings (sinusoidal, Gaussian Fourier, FiLM). This enables (i) generalization across intervals through parallel training, and (ii) sensitivity to local temporal position, ensuring bridge consistency and global coherence.

\pagebreak
\subsection{Experiments details}
The Adam \citep{kingma2017adammethodstochasticoptimization} or AdamW \citep{loshchilov2019decoupledweightdecayregularization} optimizer is used throughout experiments. Unless stated otherwise, we employ a learning rate of $2*10^{-4}$, and SiLU activations are applied on each layer.

\subsubsection[Exact OT between Gaussian mixtures]{%
  Exact OT between Gaussian mixtures%
  \texorpdfstring{\hspace{0.4em}\protect\hyperref[sec:exact_ot]{$\upuparrows$}}{}%
}
In Figure~\ref{fig:pot_exact_ot} we can see the (\textit{exact}) "glued" OT plan empirically computed with \href{https://pythonot.github.io/}{POT}. Observe how the global trajectory transports each Gaussian component of the mixture to a single other Gaussian component of the next marginal, yielding paths without any crossing. Note that the \textit{true} \textit{multi-marginal} transport plan remains unknown even in this simple Gaussian mixture setting.
\begin{figure}[h]
    \centering
    \includegraphics[width=0.4\linewidth]{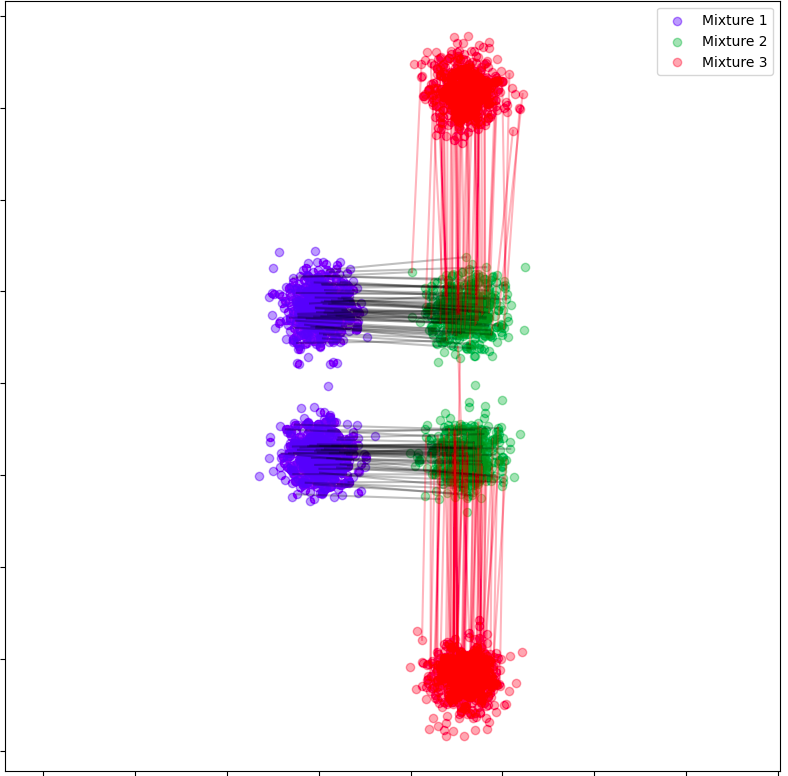}
    \caption{\small Here we computed the OT plan between each pair of adjacent marginals empirically, in red and black lines. This plan can serve as a good proxy for the true multi-marginal plan.}
    \label{fig:pot_exact_ot}
\end{figure}

\subsubsection[8Gaussians and Moons experiment]{%
  8Gaussians and Moons experiment%
  \texorpdfstring{\hspace{0.4em}\protect\hyperref[sec:moons_8gauss]{$\upuparrows$}}{}%
}
We used the same experimental setting as \citep{shi2023diffusionschrodingerbridgematching}, except that we increase the batch size proportionally to the number of intermediate bridges. The 2-Wasserstein distance are computed with \texttt{pot} and the integrated path energy are computed with $\mathbb{E}\left[\int_{0}^{T} \| v(t, \mathbf{Z}_t) \|^2 \, dt \right]$ where $Z_t$ is the process simulated along the ODE drift \ref{eq:ODE_drift}.
\begin{figure}[h]
    \centering
    \includegraphics[width=0.7\linewidth]{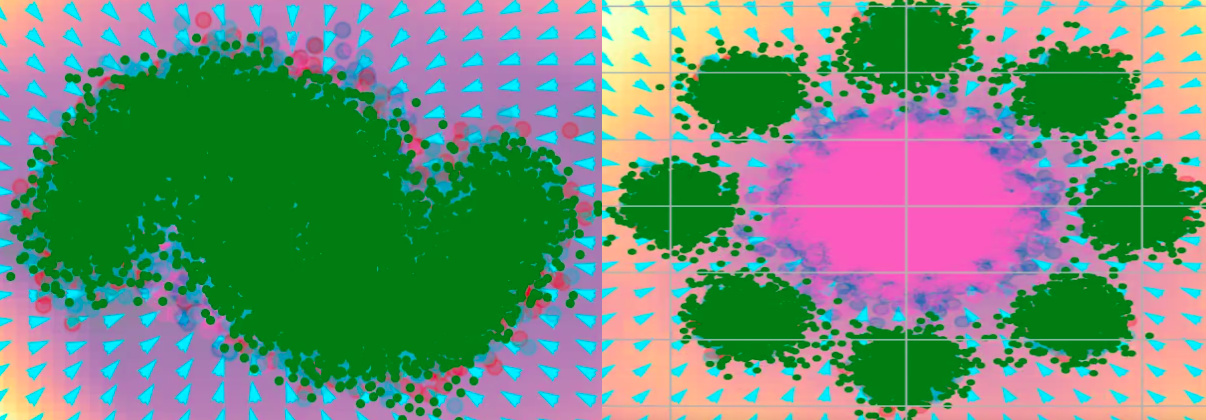}
    \caption{\small Third marginal fitting for the moons and 8-Gaussian trajectories. Blue vectors indicate the drift direction, with gradient intensity showing vector field strength; green points denote moving samples, and pink highlights the Gaussian fitted along the trajectory.}
    \label{fig:moons_8gaussian}
\end{figure}

\subsubsection[$50d$ Gaussian experiments]{%
  $50d$ Gaussian experiments%
  \texorpdfstring{\hspace{0.4em}\protect\hyperref[sec:50d_gauss]{$\upuparrows$}}{}%
}
On an NVIDIA A100 GPU, the full training took approximately 300 minutes for 30 outer iterations, each with 10,000 training steps and 20 diffusion steps per bridge.

\subsubsection[$100d$ transcriptomic experiments: Embryoid Body]{%
  $100d$ transcriptomic experiments: Embryoid Body%
  \texorpdfstring{\hspace{0.4em}\protect\hyperref[sec:100_transcriptomic]{$\upuparrows$}}{}%
}
The dataset comprises 5 timepoints (each of them being the aggregation of 2 days, from day 0 to day 24) covering the progression from a homogeneous stem-cell population toward mesoderm, endoderm, and ectoderm precursors. The Embryoid Body dataset thus constitutes a realistic and challenging testbed for Schrödinger Bridge methods, combining high dimensionality, non-Gaussian distributions, and branching lineages. We preprocessed the data following \citet{tong2020trajectorynetdynamicoptimaltransport}.

During MMtSBM training all timepoints are standardized (zero mean and unit variance). We then used two distinct settings to fairly evaluate our method: in the first one we compute metrics in the normalized space, while in the second one we denormalize the data before computing metrics in order to create a comparable setting to \citet{theodoropoulos2025momentummultimarginalschrodingerbridge}. In both cases 1000 samples were withheld from each timepoint to form a test set used for evaluating the Maximum Mean Discrepancy (MMD) and the Sliced Wasserstein Distance (SWD) between test set and generated samples. Additionally, in Table~\ref{tab:embryoid-benchmark-unnorm}, odd-indexed timepoints are fully removed from the training data.

We trained a network of about $300k$ parameters for 20 outer iterations with 20,000 inner iterations.

We show in Table \ref{tab:times} the performance advantage of our method compared to an iterative algorithm such as \citet{chen2023deepmomentummultimarginalschrodinger}.

\begin{table}[h]
    \centering
    \caption{Training and sampling times for \citet{chen2023deepmomentummultimarginalschrodinger} and MMtSBM (ours) in dimension 100.}
    \begin{tabular}{l|cc|cc}
    \hline
    & \multicolumn{2}{c|}{DSBM \citep{chen2023deepmomentummultimarginalschrodinger}} & \multicolumn{2}{c}{MMtSBM (ours)} \\
    number of marginals & 4 & 5 & 4 & 5 \\
    \hline
    Train time & 33min & 44 min & 20 min & 32 min \\
    Sampling time & 2.00s & 2.02s & 2.00 s & 2.00 s \\
    \hline
    \end{tabular}
    \label{tab:times}
\end{table}

\subsubsection[$100d$ transcriptomic experiments: MULTI]{%
  $100d$ transcriptomic experiments: MULTI%
  \texorpdfstring{\hspace{0.4em}\protect\hyperref[sec:100_transcriptomic]{$\upuparrows$}}{}%
}
\label{sec:multi_annex}
We reused the preprocessed data from \cite{tong2024simulationfreeschrodingerbridgesscore}. We do not whiten it. We conducted a minimal sweep to select the best $\sigma$ ($0.3$). The network is a simple $3$-layers MLP with around $500k$ parameters and we employ $150$ discretization time steps in total. Metrics are computed over $1k$ true test samples vs $1k$ generated samples, where these generations themselves come from the previous test marginal ($\mu_{i-1}^\text{test}$ if $i$ is the left-out time). We trained $3$ models with different seeds for each left-out time (either $t=1$ or $t=2$, corresponding to days $3$ and $4$). Our reported standard deviation is the pooled variance of the best same-hyperparameters $\sigma=0.3$ models over $2$ groups, each group corresponding to a left-out time. Other papers seem to have reported the overall variance, which we think makes less sense given the structure of the problem.
\begin{table}[h]
    \centering
    \small
    \caption{\small Per-group statistics with pooled standard deviation $s_{\text{pooled}} = \sqrt{\sum (n_i - 1) s_i^2 \;/\; \sum (n_i - 1)}$, where $n_i$ and $s_i$ are the sample size and standard deviation of each group.}
    \begin{tabular}{lccc}
        Group          & Number of runs & Mean   & Std \\
        \hline
        Leave-out \& test $t=1$ & 3     & 37.026 & 0.822 \\
        Leave-out \& test $t=2$ & 3     & 52.059 & 0.367 \\
        \hline
        Global        & 6               & 44.542 & 0.637 \\
    \end{tabular}
    \label{tab:task_stats}
\end{table}

About other methods reported in Table \ref{tab:multi-benchmark}: only I-CFM, I-MFM$_{\textrm{RBF}}$, and MMtSBM (ours) do not rely on a precomputed OT plan, be it exact or approximate. GAGA \citep{sun2025geometryawaregenerativeautoencoderswarped} performs \emph{interpolation between 2 true pinned endpoints} in the latent space of a metric-aware autoencoder trained with the true exact OT plan; we thus still claim SOTA, either within methods solving the SB, or within methods doing "pure" trajectory \emph{inference} (without a pinned true endpoint).

\subsubsection[Biotine cell culture]{%
  Biotine cell culture%
  \texorpdfstring{\hspace{0.4em}\protect\hyperref[sec:biotine]{$\upuparrows$}}{}%
}
We perform learning directly in image space at $3 \times 128 \times 128$ definition with 3M parameters UNets. We also experimented with learning in a VAE latent space but produced images were more blurry.

\begin{table}[h!]
    \centering
    \caption{Training and sampling statistics for video generation on the Biotine dataset.}
    \begin{tabular}{lr}
    \hline
    Dataset & Biotine \\
    Dimension & $128 \times 128 \times 3 = 49,152$ \\
    Number of marginals & 7 \\
    Training time & 5 h \\
    Number of epochs & 5 \\
    Sampling time & 32 s \\
    Generated frames & 602 \\
    \hline
    \end{tabular}
\end{table}

The model trains in only 5 hours and subsequently generates an entire 602-frame trajectory in just 32 seconds, demonstrating both low training cost and highly efficient sampling.

\subsubsection[KTH Actions]{%
  KTH Actions%
  \texorpdfstring{\hspace{0.4em}\protect\hyperref[sec:kth]{$\upuparrows$}}{}%
}
We perform learning directly in image space at $1\times 160 \times 120$ definition with 14M parameters purely convolutional UNets. For training we use the AdamW optimizer, $100,000$ warmup steps, a noise scale $\sigma=0.5$, a learning rate of $2e-4$, $10,000$ inner steps and $30$ outer iterations. Videos are generated with $400$ steps.

\subsection{Proofs}
\label{sec:proofs}

\subsubsection{Definition~\ref{def:factorized_reciprocal_class}}

\begin{proof}[Proof of variational proposition in Definition~\ref{def:factorized_reciprocal_class} (variational characterization)]
By the additive property of the KL divergence 
\citep{leonard2014survey}, for any $P \in \mathcal P(C([0,T],\mathbb R^d))$ 
and $\Pi \in \mathcal R^\otimes(\mathbb{Q})$, we can write
\[
KL(P \,\|\, \Pi) 
= KL(P_{t_0,\dots,t_K} \,\|\, \Pi_{t_0,\dots,t_K}) 
+ \mathbb{E}_{P_{t_0,\dots,t_K}}
  \Big[ KL\!\Big(P^{x_0,\dots,x_K}_{[0,T]} \,\|\, 
                 \bigotimes_{i=0}^{K-1} 
                 \mathbb{Q}^{x_i,x_{i+1}}_{[t_i,t_{i+1}]}\Big)\Big],
\]
where $P^{x_0,\dots,x_K}_{[0,T]}$ denotes the conditional law of $P$ given its 
values at the grid points $(t_0,\dots,t_K)$.  

Restricting to $\Pi$ such that $\Pi_{t_0,\dots,t_K} = P_{t_0,\dots,t_K}$ cancels 
the first KL term, and then the minimizer is uniquely obtained by replacing 
the conditional path law of $P$ with the tensor product of $Q$-bridges between 
each $(x_i,x_{i+1})$.  

Hence the optimal projection is
\[
\Pi^\star = P_{t_0,\dots,t_K} \;\bigotimes_{i=0}^{K-1} 
\mathbb{Q}^{x_i,x_{i+1}}_{[t_i,t_{i+1}]},
\]
which is exactly the definition of the factorized reciprocal projection. 
\end{proof}

\subsubsection{Definition~\ref{def:markov_projection_factorized}}

\begin{proof}[Proof of proposition in the Definition~\ref{def:markov_projection_factorized} in the Brownian case]
By Definition~\ref{def:markov_projection_factorized}, the effective drift is
\[
v_t^\star(x) 
= \sigma_t^2 \, \mathbb{E}_{\Pi_{t_{i+1}\mid t}}
\Big[ \nabla \log Q^{\,|t_i,t_{i+1}}_{t}(X_{t_{i+1}} \mid X_t) 
\,\Big|\, X_t = x \Big].
\]
For a Brownian reference process, the transition kernel is Gaussian,
\[
Q^{\,|t_i,t_{i+1}}_{t}(y \mid x)
= \frac{1}{(2\pi\sigma^2 (t_{i+1}-t))^{d/2}}
\exp\!\left( -\tfrac{\|y-x\|^2}{2\sigma^2 (t_{i+1}-t)} \right),
\]
so that
\[
\nabla_x \log Q^{\,|t_i,t_{i+1}}_{t}(y \mid x) 
= \frac{y-x}{\sigma^2 (t_{i+1}-t)}.
\]
Plugging this into the definition yields
\[
v_t^\star(x) 
= \sigma_t^2 \, \mathbb{E}\!\left[ \tfrac{X_{t_{i+1}} - x}{\sigma^2 (t_{i+1}-t)} 
\,\Big|\, X_t = x \right].
\]
In the Brownian case $\sigma_t^2=\sigma^2$, which simplifies to
\[
v_t^\star(x) 
= \frac{\mathbb{E}[X_{t_{i+1}} \mid X_t = x] - x}{t_{i+1}-t},
\]
as claimed.
\end{proof}

\subsubsection{Proposition~\ref{prop:existence_uniqueness}}

\begin{proof}[Proof of Proposition~\ref{prop:existence_uniqueness}]
The feasible set
\[
\mathcal A = \{ P : P \ll Q, \; P_{t_i} = \mu_{t_i}, \; i=0,\dots,K \}
\]
is convex. 
Since the functional $P \mapsto \KL(P\|Q)$ is strictly convex, there is at most one minimizer.

To show existence, observe that $\mathcal A$ is non-empty. 
Indeed, consider any coupling $\gamma$ of $(\mu_{t_0},\dots,\mu_{t_K})$. 
For each pair $(x_i,x_{i+1})$, let $Q^{x_i,x_{i+1}}_{[t_i,t_{i+1}]}$ denote the Brownian bridge of $Q$
conditioned on $X_{t_i}=x_i$ and $X_{t_{i+1}}=x_{i+1}$. 
Then the measure
\[
P = \int \bigotimes_{i=0}^{K-1} Q^{x_i,x_{i+1}}_{[t_i,t_{i+1}]} \; d\gamma(x_0,\dots,x_K)
\]
belongs to $\mathcal A$. Hence the admissible set is non-empty.

Therefore, (MMSB) admits a unique solution $P^\star$.
\end{proof}

\subsubsection{Proposition~\ref{prop:dyn_static_equiv}}

\begin{proof}[Proof of Proposition~\ref{prop:dyn_static_equiv}]\label{app:proof_dyn_static_equiv}
The argument is identical to Proposition~2.10 in \citet{leonard2014survey}, extended 
to the multi-marginal setting. For any admissible path measure $P \ll Q$, 
the additivity property of the relative entropy gives
\[
KL(P\,\|\,Q) = KL(P_{t_0,\dots,t_K}\,\|\,Q_{t_0,\dots,t_K})
+ \mathbb E_{P_{t_0,\dots,t_K}}\!\left[
KL\!\left(P(\cdot \mid X_{t_0},\dots,X_{t_K}) \,\|\, Q(\cdot \mid X_{t_0},\dots,X_{t_K})\right)
\right].
\]
Since the second term is always nonnegative, minimizing the dynamic problem 
is equivalent to minimizing the static one. Moreover, the inequality becomes 
an equality if and only if 
\[
P(\cdot \mid X_{t_0},\dots,X_{t_K}) = Q(\cdot \mid X_{t_0},\dots,X_{t_K}),
\quad P_{t_0,\dots,t_K}\text{-a.s.}
\]
Hence the optimal dynamic solution $P^\star$ is uniquely obtained from the 
optimal static solution $\pi^\star$ by gluing the conditional bridges of $Q$, 
which establishes the equivalence.
\end{proof}

\subsubsection{Proposition~\ref{prop:markovianity}}

\begin{proof}[Proof of Proposition~\ref{prop:markovianity}]\label{app:proof_markovianity}

We follow the argument of \citep[Prop.~2.10]{leonard2014survey}.  
Fix an intermediate time $t_k$ with $0<k<n$.  
For any $Q\in\mathcal P(\Omega)$ and $z\in X$, set
\[
Q^{t_k,z}_{[0,t_k]} := Q(X_{[0,t_k]}\in\cdot \mid X_{t_k}=z),\qquad 
Q^{t_k,z}_{[t_k,1]} := Q(X_{[t_k,1]}\in\cdot \mid X_{t_k}=z).
\]
Let $\mu\in\mathcal P(X)$ and for each $z\in X$ prescribe 
$Q^{<}_{z}\in\mathcal P(\Omega_{[0,t_k]}\cap\{X_{t_k}=z\})$, 
$Q^{>}_{z}\in\mathcal P(\Omega_{[t_k,1]}\cap\{X_{t_k}=z\})$.  
By the entropy additivity property (see formula (A.8) in \citet{leonard2014survey}),  
the measure
\[
P^* = \int_X Q^{<}_{z}\otimes Q^{>}_{z}\, \mu(dz)
\]
is the unique minimizer of $H(\cdot\mid R)$ under these constraints, and it satisfies
\[
P^*_{[t_k,1]}(\cdot \mid X_{[0,t_k]}) = P^*_{[t_k,1]}(\cdot \mid X_{t_k}).
\]
This is exactly the Markov property at time $t_k$.

Now apply this to $Q=\widehat P$, the solution of the multi-marginal Schrödinger problem.  
If $\widehat P$ were not Markov, one could construct a measure $P^*$ with the same 
time-marginal constraints but strictly smaller entropy, a contradiction with the definition of a minimizer.  
Since $t_k$ was arbitrary, $\widehat P$ must be Markov at all grid times 
$t_0,\dots,t_n$, hence Markov on $[0,1]$.
\end{proof}

\subsubsection{Proposition~\ref{prop:form_dynamic}}

\begin{proof}[Proof of Proposition~\ref{prop:form_dynamic}]
The argument is a direct extension of Theorem~2.8 and Proposition~2.10 
in \citet{leonard2014survey}.\label{app:proof_form_dynamic} 

Assume that the reference law $Q_{t_0,\dots,t_K}$ satisfies the usual 
regularity conditions: (i) each one-time marginal coincides with a reference 
measure $m$; (ii) there exists a nonnegative function $A$ such that 
\[
Q_{t_0,\dots,t_K}(dx_0,\dots,dx_K) 
\;\ge\; \exp\!\Big(-\sum_{i=0}^K A(x_i)\Big)\, m(dx_0)\cdots m(dx_K);
\]
(iii) there exists $B$ such that 
\[
\int_{\mathcal X^{K+1}} \exp\!\Big(-\sum_{i=0}^K B(x_i)\Big)\, 
Q_{t_0,\dots,t_K}(dx_0,\dots,dx_K) < \infty;
\]
(iv) either $m^{\otimes(K+1)} \ll Q_{t_0,\dots,t_K}$ or the converse holds. 
Suppose further that the prescribed marginals $(\pi_{t_0},\dots,\pi_{t_K})$ 
satisfy $H(\pi_{t_i}\,|\,m) < \infty$, 
\[
\sum_{i=0}^K \int (A+B)(x)\, d\pi_{t_i}(x) < \infty,
\]
and that they are internal in the sense of Proposition~2.6 of 
\citep{leonard2014survey}. 

Under these assumptions, the dual problem is well posed. Introducing 
Lagrange multipliers $(\varphi_i)_{i=0}^K$ for the marginal constraints, 
convex duality shows that the minimizer $\pi^\star$ of the static problem is 
absolutely continuous with respect to $Q_{t_0,\dots,t_K}$ with density
\[
\frac{d\pi^\star}{dQ_{t_0,\dots,t_K}}(x_0,\dots,x_K) 
= \exp\!\Big(\sum_{i=0}^K \varphi_i(x_i)\Big).
\]
Defining $f_i(x_i) := e^{\varphi_i(x_i)}$ yields the factorized form
\[
\frac{d\pi^\star}{dQ_{t_0,\dots,t_K}}(x_0,\dots,x_K) 
= \prod_{i=0}^K f_i(x_i).
\]
\end{proof}

\subsubsection{Proposition~\ref{prop:var_proj_factorized}}
\begin{proof}[Proof of Proposition~\ref{prop:var_proj_factorized}]
The argument is the same as in the two-marginal case 
(\citealp{shi2023diffusionschrodingerbridgematching}, Prop.~2), except that all
computations must now be performed interval by interval along the grid
$t_0<\dots<t_K$.
Under Assumptions~A1--A3, the Doob--$h$ transform is well-defined on each interval
$[t_i,t_{i+1}]$ and Lemma~11 of 
\cite{shi2023diffusionschrodingerbridgematching} applies verbatim.  
The only change is that the terminal conditioning in the backward equation is at 
$t_{i+1}$ instead of $T$.  
This yields the drift
\[
v_t^\Pi(x)
=
\sigma_t^2\,
\mathbb{E}_\Pi\!\left[
\nabla \log Q^{\,t_i,t_{i+1}}_{t_{i+1}\mid t}
   (X_{t_{i+1}}\mid X_t)
\,\big|\, X_{t_i}, X_t 
\right],
\qquad t\in[t_i,t_{i+1}].
\]
Hence the dynamics of~$\Pi$ is piecewise independent: its increment on
$[t_i,t_{i+1}]$ depends only on the local bridge~$Q^{t_i,t_{i+1}}$.

The same interval-wise independence holds for any Markov $M\in\mathcal M$,
whose SDE also factorizes on the grid.  
Thus both $\Pi$ and $M$ have product decompositions over the intervals, and their 
Radon--Nikodym derivative factorizes multiplicatively,
\[
\frac{d\Pi}{dM}
=
\prod_{i=0}^{K-1}
\frac{d\Pi^{(i)}}{dM^{(i)}}.
\]
Taking logarithms and integrating with respect to $\Pi$ gives
the additivity of the relative entropy,
\[
KL(\Pi\|M)
=
\sum_{i=0}^{K-1}
KL(\Pi^{(i)}\|M^{(i)}).
\]

For each interval $[t_i,t_{i+1}]$, using the conditional expectation identity
as in the proof of \cite{shi2023diffusionschrodingerbridgematching}, we have
for every $t\in[t_i,t_{i+1}]$,
\[
\mathbb{E}_{\Pi_{t_{i},t}}
\!\left[
\big\|
 \sigma_t^2 \,
 \mathbb{E}_{\Pi_{t_{i+1}\mid t_i,t}}
   [\nabla\log Q^{\,t_i,t_{i+1}}_{t_{i+1}\mid t}(X_{t_{i+1}}\mid X_t)
    \mid X_t, X_{t_i}]
 - v_t(X_t)
\big\|^2
\right]
\]
\[
\ge 
\mathbb{E}_{\Pi_{t_{i},t}}
\!\left[
\big\|
 \sigma_t^2 
 \mathbb{E}_{\Pi_{t_{i+1}\mid t}}
   [\nabla\log Q^{\,t_i,t_{i+1}}_{t_{i+1}\mid t}(X_{t_{i+1}}\mid X_t)
    \mid X_t, X_{t_i}]
 - v_t^\star(X_t)
\big\|^2
\right],
\]
where the optimal drift is defined by the orthogonal projection
\[
v_t^\star(x)
=
\sigma_t^2\,
\mathbb{E}_{\Pi_{t_{i+1}\mid t}}
\!\left[
\nabla\log Q^{\,t_i,t_{i+1}}_{t_{i+1}\mid t}(X_{t_{i+1}}\mid X_t)
\mid X_t=x_t
\right].
\]

Using \cite{leonard2012}, Theorem~2.3 on each interval and summing the contributions gives
\[
KL(\Pi\|M^\star)
=
\frac12
\sum_{i=0}^{K-1}
\int_{t_i}^{t_{i+1}}
\mathbb{E}_{\Pi_t}
\!\left[
\|v_t^\Pi(X_t)-v_t^\star(X_t)\|^2/\sigma_t^2
\right] dt.
\]
Finally, the same Fokker--Planck uniqueness argument as in
\cite{shi2023diffusionschrodingerbridgematching} ensures that 
$M^\star_t=\Pi_t$ for all $t\in[t_i,t_{i+1}]$ and all $i$.  
Since the grid points are included, this implies $M^\star=\Pi$, which concludes the
proof.
\end{proof}

\subsubsection{Lemma~\ref{lem:pythagoras_factorized}}

\begin{proof}[Proof of Lemma~\ref{lem:pythagoras_factorized}]
For the Markovian part, the equality follows analogously to the proof of \citep{shi2023diffusionschrodingerbridgematching}.

For each interval $[t_i,t_{i+1}]$, the same quadratic expansion gives
\[
2\,KL\!\left(\Pi^{(i)}\,\middle\|\,M^{(i)}\right)
=
2\,KL\!\left(\Pi^{(i)}\,\middle\|\,\proj_{\mathcal M}(\Pi)^{(i)}\right)
+
2\,KL\!\left(\proj_{\mathcal M}(\Pi)^{(i)}\,\middle\|\,M^{(i)}\right).
\]
Summing this identity over $i=0,\dots,K-1$, using the interval-wise independence, yields

\[
2\,KL(\Pi\|M)
=
2\,KL\!\left(\Pi\,\middle\|\,\proj_{\mathcal M}(\Pi)\right)
+
2\,KL\!\left(\proj_{\mathcal M}(\Pi)\,\middle\|\,M\right),
\]
which is the desired result.

For the factorized reciprocal part : 

Let $\Pi \in \mathcal{R}^\otimes(Q)$ and denote by 
\[
\Pi^\star = \proj_{\mathcal{R}^\otimes(Q)}(\mathbb{P})
= \mathbb{P}_{t_0,\dots,t_K}\,\otimes_{i=0}^{K-1} Q^{x_i,x_{i+1}}_{[t_i,t_{i+1}]}.
\]
We have the Radon--Nikodym factorization
\[
\frac{d\mathbb{P}}{d\Pi} 
= \frac{d\mathbb{P}}{d\Pi^\star} \cdot \frac{d\Pi^\star}{d\Pi}(X_{t_0},\dots,X_{t_K}).
\]
By integrating w.r.t.\ $\mathbb{P}$ and applying Csiszár's Pythagorean identity 
\citep[Eq.~2.6]{Csiszar1975}, we obtain
\[
KL(\mathbb{P}\|\Pi) 
= KL(\mathbb{P}\|\Pi^\star) 
+ \int \log \frac{d\Pi^\star}{d\Pi}(x_0,\dots,x_K)\, d\mathbb{P}_{t_0,\dots,t_K}.
\]
Since $\mathbb{P}_{t_0,\dots,t_K} = \Pi^\star_{t_0,\dots,t_K}$, the second term equals 
\[
\int \log \frac{d\Pi^\star}{d\Pi}(x_0,\dots,x_K)\, d\Pi^\star_{t_0,\dots,t_K}
= KL(\Pi^\star\|\Pi).
\]
Thus
\[
KL(\mathbb{P}\|\Pi) 
= KL(\mathbb{P}\|\Pi^\star) + KL(\Pi^\star\|\Pi),
\]
which concludes the proof.
\end{proof}

\subsubsection{Proposition~\ref{prop:reverse_sde}}

\begin{proof}[Proof of Proposition~\ref{prop:reverse_sde}]
\label{proof:proff_reverse_sde}
It follows from the fact that the time-reversal map 
$\mathcal{T}:\Omega \to \Omega$ is a bijection, and by reversibility of the 
reference process $\mathbb{Q}$ we have, for any probability measure 
$\mathbb{P}\in\mathcal{P}(C)$,
\[
KL(\mathbb{P}\,\|\,\mathbb{Q})
= KL(\mathbb{P}\circ \mathcal{T} \,\|\, \mathbb{Q}\circ \mathcal{T})
= KL(\mathbb{P}\circ \mathcal{T} \,\|\, \mathbb{Q}).
\]

To prove the direction “$\Longrightarrow$”, assume 
$\mathbb{P}\in\mathcal{R}^\otimes(\mathbb{Q})$ is the minimizer of the 
forward problem. Then, for any $\Pi\in \mathcal{R}^\otimes(\mathbb{Q})$ 
we have $\Pi\circ \mathcal{T} \in \mathcal{R}^\otimes(\mathbb{Q})$, and
\[
KL(\Pi \,\|\, \mathbb{Q}) 
= KL(\Pi\circ \mathcal{T} \,\|\, \mathbb{Q}\circ \mathcal{T})
\;\;\geq\;\; KL(\mathbb{P}\circ \mathcal{T} \,\|\, \mathbb{Q}\circ \mathcal{T})
= KL(\mathbb{P}\,\|\, \mathbb{Q}).
\]

The reverse direction follows by symmetry, replacing 
$\mathbb{P}$ with $\mathbb{P}\circ \mathcal{T}$. 
Thus, working with forward or backward processes is equivalent up to 
the bijection $\mathcal{T}$, and the KL minimization problem is unchanged. 
In particular, this justifies that alternating forward and backward 
steps in the IMFF algorithm is well-defined and analogous to IPF.
\end{proof}

\subsubsection{Proposition~\ref{prop:convergence_projection}}

\begin{proof}[Proof of Proposition~\ref{prop:convergence_projection}, first claim]
As a reminder, we follow the same argument as in 
\citep{shi2023diffusionschrodingerbridgematching} and 
\citep{debortoli2021diffusionschrodingerbridgeapplications}. 
Applying Lemma~\ref{lem:pythagoras_factorized}, for any $N \in \mathbb{N}$ we obtain
\[
KL(\mathbb{P}^0 \,\|\, \mathbb{P}^\star) 
= KL(\mathbb{P}^0 \,\|\, \mathbb{P}^1) + KL(\mathbb{P}^1 \,\|\, \mathbb{P}^2) + \dots 
+ KL(\mathbb{P}^N \,\|\, \mathbb{P}^\star).
\]
Since each term is nonnegative, we deduce the monotonicity
\[
KL(\mathbb{P}^{n+1} \,\|\, \mathbb{P}^\star) \;\leq\; KL(\mathbb{P}^n \,\|\, \mathbb{P}^\star),
\]
and boundedness 
\(
KL(\mathbb{P}^n \,\|\, \mathbb{P}^\star) \leq KL(\mathbb{P}^0 \,\|\, \mathbb{P}^\star) < \infty.
\) 
This proves the first claim.
\end{proof}

\begin{proof}[Proof of Proposition~\ref{prop:convergence_projection}, second claim]
We proceed by induction, adapting the argument of 
\citep[Appendix C.8]{debortoli2021diffusionschrodingerbridgeapplications}.  

At initialization, we choose $\mathbb{P}^0 \in \mathcal{R}^\otimes(\mathbb{Q})$ with 
$\mathbb{P}^0_{t_i} = \mu_{t_i}$ for all $i$.  
We also define $M^0 = \proj_{\mathcal{M}}(\mathbb{P}^0)$.  

By construction (Algorithm~\ref{alg:IMFF}), the IMFF sequence alternates:
\[
\mathbb{P}^{2n+1} = \proj_{\mathcal{M}}(\mathbb{P}^{2n}), 
\qquad 
\mathbb{P}^{2n+2} = \proj_{\mathcal{R}^\otimes(\mathbb{Q})}(\mathbb{P}^{2n+1}).
\]

Suppose now that $\mathbb{P}^{2n}$ satisfies the claim.  
By definition, $\mathbb{P}^{2n+1} \in \mathcal{M}$ and 
$\mathbb{P}^{2n+2} \in \mathcal{R}^\otimes(\mathbb{Q})$.  
From Lemma~\ref{lem:pythagoras_factorized}, we then have
\[
KL(\mathbb{P}^{2n+1}\,\|\,P^\star) \leq KL(\mathbb{P}^{2n}\,\|\,P^\star), 
\qquad 
KL(\mathbb{P}^{2n+2}\,\|\,P^\star) \leq KL(\mathbb{P}^{2n+1}\,\|\,P^\star).
\]

Hence, $(KL(\mathbb{P}^n \,\|\, P^\star))_{n \in \mathbb{N}}$ is a nonincreasing sequence bounded below by $0$, and is therefore convergent.  
Moreover, by induction we have $\mathbb{P}^n \in \mathcal{M} \cap \mathcal{R}^\otimes(\mathbb{Q})$ for all $n$, so the limit must coincide with $P^\star$, the unique measure in this intersection with prescribed marginals.  

Finally, note that in Algorithm~\ref{alg:IMFF} the forward and backward Markovian steps are time-reversals of each other (they follow the same law under the change of variable $t \mapsto T-t$).  
Therefore, alternating a backward step with a forward reciprocal projection, or a forward step with a backward reciprocal projection, is equivalent from the viewpoint of convergence analysis.  
All the arguments above apply symmetrically in both directions, and the resulting sequence $(\mathbb{P}^n)_{n\in\mathbb{N}}$ still converges.  

We conclude that
\[
\lim_{n \to \infty} KL(\mathbb{P}^n \,\|\, P^\star) = 0,
\]
and $P^\star$ is indeed the weak solution produced by the IMFF algorithm, proving the second claim.
\end{proof}

\subsubsection{Theorem~\ref{thm:imff_convergence}}

\begin{proof}[Proof of Theorem~\ref{thm:imff_convergence}]
As a reminder, the argument is the same as in 
\citep{shi2023diffusionschrodingerbridgematching} and 
\citep{debortoli2021diffusionschrodingerbridgeapplications}, 
but adapted to the multi-marginal setting. 

By Proposition~\ref{prop:convergence_projection}, the sequence 
$(\mathbb{P}^n)_{n \in \mathbb{N}}$ is bounded in KL divergence with respect to 
$\mathbb{P}^\star$, hence relatively compact under weak convergence. 
Thus, it admits a subsequence $(\mathbb{P}^{n_j})_j$ converging weakly to some 
limit $\mathbb{P}^\infty$. By construction, $\mathbb{P}^\infty \in \mathcal{M} \cap \mathcal{R}^\otimes(Q)$ 
and matches the marginals $(\mu_{t_i})_{i=0}^K$, so by uniqueness of the weak MMSB solution 
we must have $\mathbb{P}^\infty = \mathbb{P}^\star$. 

By lower semicontinuity of KL, this implies
\[
\lim_{n \to \infty} KL(\mathbb{P}^n \,\|\, \mathbb{P}^\star) = 0.
\]

Finally, the inequality
\[
KL(\mathbb{P}^{\mathrm{MMSB}} \,\|\, Q) 
\;\leq\; KL(\mathbb{P}^\star \,\|\, Q) 
\;\leq\; KL(\mathbb{P}^{\mathrm{pair}} \,\|\, Q)
\]
is justified because $\mathbb{P}^{\mathrm{MMSB}}$ is the global minimizer (hence gives the smallest KL), 
while $\mathbb{P}^\star$ is the best Markovian candidate in 
$\mathcal{M} \cap \mathcal{R}^\otimes(Q)$, and therefore lies below the pairwise construction 
obtained by gluing local bridges. 
\end{proof}

\end{document}